\documentclass{article}

\usepackage[hyphens]{url} 
\usepackage{graphicx}
\usepackage[round]{natbib}

\usepackage{caption}

\usepackage{newfloat}
\usepackage{listings}

\usepackage{hyperref}

\usepackage{booktabs}       
\usepackage{amsfonts}       
\usepackage{nicefrac}      
\usepackage{microtype}      
\usepackage{xcolor}         

\usepackage{titletoc}
\usepackage{dsfont}
\usepackage{amsmath}
\usepackage{amssymb}
\usepackage{amsthm}
\usepackage{algpseudocode}
\usepackage[algo2e,ruled,linesnumbered]{algorithm2e}
\usepackage{subcaption}
\usepackage{array}

\DeclareMathOperator*{\argmax}{arg\,max}
\allowdisplaybreaks

\newtheorem{theorem}{Theorem}

\newtheorem{corollary}{Corollary}
\newtheorem{lemma}{Lemma}
\newtheorem{Fact}{Fact}
\newtheorem{remark}{Remark}
\newtheorem{assumption}{Assumption}
\newtheorem{definition}{Definition}

\newcommand{\Xc}{\mathcal{X}}

\newcommand{\Rb}{\mathbb{R}}
\newcommand{\Eb}{\mathbb{E}}
\newcommand{\Ac}{\mathcal{A}}

\newcommand{\Dc}{\mathcal{D}}
\newcommand{\Nb}{\mathbb{N}}
\newcommand{\Gc}{\mathcal{G}}

\newcommand{\xtil}{\tilde{x}}

\newcommand{\Econf}{\mathcal{E}_{\mathrm{conf}}}

\newcommand{\rtil}{\tilde{r}}

\newcommand{\Bb}{\mathbb{B}}
\newcommand{\Sb}{\mathbb{S}}

\newcommand{\Pb}{\mathbb{P}}
\newcommand{\Yc}{\mathcal{Y}}

\newcommand{\Uc}{\mathcal{U}}

\newcommand{\Oc}{\mathcal{O}}

\newcommand{\Octil}{\tilde{\mathcal{O}}}
\newcommand{\Hc}{\mathcal{H}}
\newcommand{\tone}{\mathrm{Term\ I}}
\newcommand{\ttwo}{\mathrm{Term\ II}}
\newcommand{\ttwoa}{\mathrm{Term\ II.A}}
\newcommand{\ttwob}{\mathrm{Term\ II.B}}
\newcommand{\tthree}{\mathrm{Term\ III}}

\newcommand{\Pc}{\mathcal{P}}
\newcommand{\Nc}{\mathcal{N}}

\newcommand{\btil}{\tilde{b}}

\newcommand{\indic}{\mathds{1}}

\newcommand{\gray}[1]{\textcolor{gray}{#1}}

\newcommand{\polylog}{\mathrm{polylog}}

\begin{document}

\title{Directional Optimism for Safe Linear Bandits}
\date{}
\author{Spencer Hutchinson \and Berkay Turan \and Mahnoosh Alizadeh\\
\and
 University of California, Santa Barbara
 }

 \maketitle

\begin{abstract}
  The safe linear bandit problem is a version of the classical stochastic linear bandit problem where the learner's actions must satisfy an uncertain constraint at all rounds.
  Due its applicability to many real-world settings, this problem has received considerable attention in recent years.
  By leveraging a novel approach that we call \emph{directional optimism}, we find that it is possible to achieve improved regret guarantees for both well-separated problem instances and action sets that are finite star convex sets.
  Furthermore, we propose a novel algorithm for this setting that improves on existing algorithms in terms of empirical performance, while enjoying matching regret guarantees.
  Lastly, we introduce a generalization of the safe linear bandit setting where the constraints are convex and adapt our algorithms and analyses to this setting by leveraging a novel convex-analysis based approach.
\end{abstract}

\section{Introduction}

The stochastic linear bandit setting \cite{dani2008stochastic, rusmevichientong2010linearly,abbasi2011improved} is a sequential decision-making problem where, at each round, a learner chooses a vector action and subsequently receives a reward that, in expectation, is a linear function of the action.
This problem has found broad applications in fields ranging from online recommendation engines to ad placement systems, to clinical trials.
In the rich literature that has emerged, it is often assumed that any constraints on the learner's actions are \emph{known}.
In the real world, however, there are often constraints that are both \emph{uncertain} and need to be met at \emph{all rounds}, such as toxicity limits in clinical trials or sensitive topics for recommendation engines.
As a result, linear bandit problems with uncertain and roundwise constraints have received considerable attention in recent years from works such as \citet{amani2019linear}, \citet{khezeli2020safe}, \citet{pacchiano2021stochastic}, \citet{moradipari2021safe} and \citet{varma2023stochastic}.

A natural formulation of the safe linear bandit problem, initially studied by \citet{moradipari2021safe}, imposes a linear constraint on every action $x_t$ of the form $a^\top x_t \leq b$ where $a$ is unknown, $b$ is known and the learner gets noisy feedback on $a^\top x_t$.
To address this problem, algorithms have been proposed with $\Octil(d^{3/2}\sqrt{T})$ (\cite{moradipari2021safe}) and $\Octil(d\sqrt{T})$ (\cite{pacchiano2021stochastic,amani2021decentralized}) regret.
These algorithms operate by choosing actions from a pessimistically safe set using versions of Thompson sampling or upper confidence bound where the confidence set for the reward parameter is scaled by a \emph{fixed} constant in both cases.
In this work, we introduce an algorithm with matching $\Octil(d\sqrt{T})$ regret that avoids the use of this fixed scaling by implementing optimism with respect to \emph{directions}, and find that, when compared to the above-mentioned approach, our algorithm enjoys improved performance in problem instances with less restrictive constraints.
Leveraging this intuition, we then give algorithms that enjoy improved regret guarantees in terms of the problem dimension for well-separated problem instances and settings with finite star convex action sets.

\begin{table*}[t] 
  \centering
  \resizebox{\textwidth}{!}{
  \begin{tabular}{p{0.15\textwidth}p{0.18\textwidth}p{0.18\textwidth}p{0.18\textwidth}p{0.20\textwidth}}\toprule
    Algorithm & General  & Problem-dependent & Finite-star convex action set & Linked convex constraint\\
    \midrule
    \gray{Safe-LTS [1]} & \gray{$ \Octil ( d^{3/2} \sqrt{T} ) $} & - & - & - \\
    \midrule
    \gray{GenOP [2],[3]} & \gray{$ \Octil (  d \sqrt{T} )$ } & $\Octil\left(\frac{d^2}{\Delta} + \sqrt{T} \right)$ \mbox{Appendix \ref{apx:polb}} & - & $ \Octil (  d \sqrt{T} )$ \mbox{Appendix \ref{apx:convex}}\\
    \midrule
    ROFUL (Alg.~\ref{alg:main_alg}) & $ \Octil (  d \sqrt{T} )$ \mbox{Theorem~\ref{thm:main}} & $\Octil\left(\frac{d^2}{\Delta} + \sqrt{T} \right)$ \mbox{Corollary~\ref{thm:prob_dep_reg}} & - &$ \Octil (  d \sqrt{T} )$ \mbox{Appendix \ref{apx:convex}} \\
    \midrule
    Safe-PE (Alg.~\ref{alg:pe1}) & - & - & $\Octil (  \sqrt{d T} )$ \mbox{Theorem~\ref{thm:safepe}} & $\Octil (  \sqrt{d T} )$ \mbox{Appendix \ref{apx:convex}}  \\
    \bottomrule
  \end{tabular}
  }
  \caption{Algorithms and regret bounds developed in existing works and this paper for the safe linear bandit problem where $T$ is horizon and $d$ is problem dimension. 
  Existing work is shown in \gray{gray}, where the references are [1] for \cite{moradipari2021safe}, [2] for \cite{pacchiano2021stochastic} and [3] for \cite{amani2021decentralized}.
  Due to variations in problem settings in existing work, we use the name GenOP to refer to a generic upper confidence bound-based algorithm that uses the expanded confidence set approach from [2], [3] (see Section \ref{sec:comp} for details).
  }
  \label{tbl:comp}
  % \vspace{-0.05in}
\end{table*}

In fact, this approach is part of a broader perspective for the safe linear bandit problem in which we understand this setting as fundamentally a problem of choosing directions (rather than actions).
Since the set of feasible actions is unknown in safe linear bandits, the uncertainty in the problem comes from both the uncertainty in the reward and the uncertainty in the diameter of the feasible action set in each direction.
Accordingly, any algorithm for this setting should appropriately quantify both of these uncertainties to ensure low regret.
This understanding facilitates our contributions in both geometry-dependent regret guarantees and empirical performance.

\paragraph{Contributions}

Our contributions are summarized in Table \ref{tbl:comp} and in the following:
\begin{itemize}
  \item We propose a novel UCB-based algorithm, ROFUL, which enjoys $\Octil \left( d \sqrt{T} \right)$ regret.
  We provide some intuition and empirical evidence as to when ROFUL is preferred over existing approaches. (Section \ref{sec:restopt})
  \item  We introduce a notion of well-separated problem instance in safe linear bandits, and show that it is possible to achieve $\Octil\left(\frac{d^2}{\Delta} + \sqrt{T} \right)$ regret in this setting. (Section \ref{sec:prob_dep})
  \item We study the case when the action set is a finite star convex set and introduce a phased elimination-based algorithm, Safe-PE, which is proven to enjoy $\Octil \left( \sqrt{d T} \right)$ regret. (Section \ref{sec:spe})
  \item We introduce a generalization of the safe linear bandit problem, which we call \emph{linked convex constraints}, where each action $x_t$ needs to satisfy $A x_t \in \Gc$ for all $t \in [T]$ where $\Gc$ is an arbitrary convex set. We extend the ROFUL and Safe-PE algorithms and their analyses to this setting with a novel convex analysis-based approach. (Section~\ref{sec:conv})
  \item Simulation results provide validation for the theoretical guarantees and numerical comparison to existing approaches. (Section \ref{sec:num_lin})
\end{itemize}

\paragraph{Related Work}

Uncertain constraints have been considered in various learning and optimization problems, often under the umbrella of ``safe learning''.
This includes constrained Markov decision processes (CMDP), where the constraints take the form of limits on auxiliary cost functions (\citet{achiam2017constrained,wachi2020safe,liu2021learning,amani2021safe,bura2022dope,lindner2023learning}).
There have also been works that study convex optimization with uncertain constraints that are linear (\citet{usmanova2019safe,fereydounian2020safe}), and safe bandit optimization with Gaussian process priors on the objective and constraints (\citet{sui2015safe,sui2018stagewise}).
Although the Gaussian process bandit framework is able to capture a wider class of reward and constraint functions than linear bandits, safe Gaussian process bandit works typically make the stronger assumption that the constraint is not tight on the optimal action.
Some recent literature has also studied safe exploration of bandits (\citet{wang2022best}) as well as best arm identification under safety constraints (\citet{wan2022safe,lindner2022interactively,camilleri2022active}).
These works consider objectives other than regret minimization, i.e. accurate estimation of policy value or finding the best arm, and are therefore distinct from the regret minimization setting that we study here.

For the bandit setting in particular, various types of constraints have been considered, including knapsacks, cumulative constraints and conservatism constraints.
In knapsack bandits, pulling each arm yields both a reward and a resource consumption with the objective being to maximize the reward before the resource runs out (\citet{badanidiyuru2013bandits,badanidiyuru2014resourceful,agrawal2016linear,agrawal2016efficient,cayci2020budget}).
There are also works that consider various types of cumulative constraints on the actions, including ones with fairness constraints (\citet{joseph2018meritocratic,grazzi2022group}), budget constraints (\citet{combes2015bandits,wu2015algorithms}) and general nonlinear constraints for which the running total is constrained (\citet{liu2021efficient}).
Similarly, there are works that bound the cumulative constraint violation in the multi-armed (\citet{chen2022strategies}) and linear (\citet{chen2022doubly}) settings.
Similar to us, \citet{chen2022doubly} uses an optimistic action set, although their algorithm does not ensure constraint satisfaction at each round and instead aims for sublinear constraint violation.
These types of cumulative constraints differ from the setting we study, where constraints are roundwise and must hold at each round.
In the conservative bandit literature, the running total of the reward needs to stay close to the baseline reward (\citet{wu2016conservative, kazerouni2017conservative}).

Various works have also studied linear bandits with roundwise constraints.
In particular, \citet{amani2019linear} studies a stochastic linear bandit setting with a linear constraint, where the constraint parameter is the linearly transformed reward parameter and there is no feedback on the constraint value.
Also, \citet{khezeli2020safe} and \citet{moradipari2020stage} study a conservative bandit setting where the reward at each round needs to stay close to a baseline.
\citet{pacchiano2021stochastic} studies a setting where the learner chooses a distribution over the actions in each round and the constraint needs to be satisfied in expectation with respect to this distribution.
Although this is a slightly different type of constraint than we consider, their approach can be adapted to our setting which we discuss further in Section~\ref{sec:comp}.

Most relevantly, several works have studied linear bandit problems with an auxiliary constraint function that the learner observes noisy feedback of and needs to ensure is always below a threshold.
\citet{moradipari2021safe} studied such a setting with a linear constraint function and proposed the Safe-LTS algorithm.
Also, \citet{amani2021decentralized} studied a decentralized version of the same problem where the agents collaborate over a communication network.
Lastly, the recent work by \citet{varma2023stochastic} considers a safe linear bandit problem where different constraints apply to different parts of the domain and the learner only receives feedback on a given constraint when she selects an action from the applicable part of the domain.
However, all of these works use algorithms that choose actions from a pessimistic action set using either linear UCB or linear TS with a confidence set that is scaled by a fixed constant, which significantly differs from our proposed algorithms as detailed in Section~\ref{sec:comp}.
Also, they do not achieve improved regret guarantees for well-separated and finite star-convex settings as we do.

\section{Preliminaries}

\label{sec:pset}

\paragraph{Notation}

We use $\Oc( \cdot )$ to refer to big-O notation and $\Octil( \cdot )$ for the same except ignoring $\log$ factors. 
To refer to the p-norm ball and sphere of radius one, we use the notation $\Bb_p$ and $\Sb_p$ respectively, where $\Bb$ and $\Sb$ refers to the 2-norm ball and sphere.
For some $n \in \Nb$, we use $[n]$ to refer to the set $\{1,2,...,n\}$.
For a matrix $M$, its transpose is denoted by $M^\top$.
For a positive definite matrix $M$ and vector $x$, the notation for the weighted norm is $\| x \|_M = \sqrt{x^\top M x}$.
For a real number $x$, the ceiling function is denoted by $\lceil x \rceil$.

\paragraph{Problem Setup}

We study a stochastic linear bandit problem with a constraint that must be satisfied at all rounds (at least with high probability).
At each round $t \in [T]$, the learner chooses an action $x_t$ from the closed set $\Xc$.
She subsequently receives the reward $y_t = \theta^\top x_t + \epsilon_t$ and the noisy constraint observation $z_t = a^\top x_t + \eta_t$, where the reward vector $\theta \in \Rb^d$ and constraint vector $a \in \Rb^d$ are unknown, and $\epsilon_t$ and $\eta_t$ are noise terms.
Critically, the learner must ensure that $a^\top x_t \leq b$ for all $t \in [T]$, where $b > 0$ is known.
We will refer to the feasible set of actions as $\Yc := \{ x \in \Xc : a^\top x \leq b \}$.

In addition to guaranteeing constraint satisfaction, the learner also aims to minimize the pseudo-regret,
\begin{equation*}
  R_T := \sum_{t=1}^T \theta^\top \left( x_* - x_t \right),
\end{equation*}
where $x_* = \argmax_{x \in \Yc} \theta^\top x$ is the optimal constraint-satisfying action.
Going forward, we will use the term regret to refer to pseudo-regret.

We use the following assumptions.
\begin{assumption}
\label{ass:set_bound}
 The action set $\Xc$ is star-convex.
 Also, it holds that $\| x \| \leq 1$ for all $x \in \Xc$ and that $\theta^\top x_* > 0$.
\end{assumption}

\begin{assumption}
\label{ass:bounded}
    There exists positive real numbers $S_a$ and $S_\theta$ such that $\| a \| \leq S_a$ and $\| \theta \| \leq S_\theta$. Let $S := \max(S_a, S_\theta)$.
    Also, it holds that $\nu := \frac{b}{S_a} \leq 1$.
\end{assumption}

\begin{remark}
  If $\nu > 1$, then it is known that the constraint is loose and therefore the problem can be treated as a conventional linear bandit problem.\footnote{If $\nu > 1$, then for all $x \in \Xc$ it holds that $a^\top x \leq \| a \| \| x \| \leq S_a \nu < b$ given that $\| x \| \leq 1 < \nu$ for all $x \in \Xc$. }
  Therefore, our assumption that $\nu \leq 1 $ avoids this trivial setting and allows for cleaner presentation of results.
\end{remark}

\paragraph{Technical Approach}

Our approach to this problem is based on the perspective that it is fundamentally a problem of choosing \emph{directions} rather than actions and therefore any solution approach should be focused on choosing directions that will result in low regret.
This perspective comes from the understanding that the only viable solutions are actions that are in the maximally-scaled part of the feasible set (i.e. the set of $x \in \Yc$ such that $\zeta x \not\in \Yc$ for all $\zeta > 1$).\footnote{To see that the optimal action must be in the maximally scaled part of the set, suppose that it is not (i.e. that there exists $\zeta > 1$ such that $\zeta x_* \in \Yc$). It follows that the point $\zeta x_*$ has larger reward than $x_*$, i.e. $ \theta^\top (\zeta x_*) > \theta^\top x_*$, and therefore $x_*$ cannot be the optimal action (where we use $\theta^\top x_* > 0$ from Assumption \ref{ass:set_bound}).}
Therefore, the challenge lies in identifying the optimal direction given that the maximum scaling of this direction is the only viable solution in that direction.
Unlike the conventional linear bandit setting, however, the feasible set is unknown and therefore the uncertainty in the problem comes from both the uncertain reward parameter \emph{and} the uncertainty in the maximum scaling in each direction (i.e. $\zeta = \max\{ \alpha \geq 0 : \alpha u \in \Yc \}$ for each unit vector $u \in \Sb$).
As such, our solutions to the problem will aim to explicitly characterize both these uncertainties in order to choose directions that will result in low regret.
This will be realized via both an upper confidence bound-based algorithm (Section \ref{sec:restopt}) and a phased elimination-based algorithm (Section \ref{sec:spe}) which are each suited for different action set geometries.

\section{Restrained Optimism Algorithm}

\label{sec:restopt}

\IncMargin{1em}
\begin{algorithm2e}[t]
\caption{Restrained OFUL (ROFUL)}
\label{alg:main_alg}
\DontPrintSemicolon
\KwIn{$\Xc,\nu,b,\beta_t,\delta \in (0,1), \lambda \geq 1$}
\For{$t=1$ \KwTo $T$}{
    Update $\hat{a}_t := V_t^{-1} \sum_{k=1}^{t-1} x_k z_k$ and $\hat{\theta}_{t} := V_t^{-1} \sum_{k=1}^{t-1} x_k y_{k}$, where $V_t =\sum_{k=1}^{t-1} x_k x_k^\top + \lambda I$. \;
    Update $\Yc_t^p := \left\{ x \in \Xc  : \hat{a}_t^\top x + \beta_t \| x \|_{V_t^{-1}} \leq b \right\}$ and $\Yc_t^o := \left\{ x \in \Xc  : \hat{a}_t^\top x - \beta_t \| x \|_{V_t^{-1}} \leq b \right\}$. \label{lne:actsets}\;
    Find a $\xtil_t \in \argmax_{x \in \Yc_t^o} \left( \hat{\theta}_{t}^\top x + \beta_{t} \| x \|_{V_{t}^{-1}} \right)$. \label{lne:opt_act} \;
    Set $\btil_t = \begin{cases} \min\left(  \frac{ \nu}{\|\xtil_t\|} ,1 \right) & \mathrm{if} \ \xtil_t \neq \mathbf{0},\\
    1 & \mathrm{else.}
    \end{cases}$ \label{lne:pess_scale}\;
    Set $\mu_t = \max \left\{\mu \in \left[ 0, 1 \right] : \mu \xtil_t \in \Yc_t^p \right\}$ and $\gamma_t = \max \left( \btil_t , \mu_t \right)$. \label{lne:scale} \;
    Play $x_t = \gamma_t \xtil_t $ and observe $y_t,z_t$. \label{lne:act}\;
}
\end{algorithm2e}
\DecMargin{1em}

In this section, we first propose the algorithm \emph{Restrained Optimism in the Face of Uncertainty for Linear bandits} (ROFUL, Algorithm \ref{alg:main_alg}) to address the stated problem, and then provide general and problem-dependent regret analyses for ROFUL in Sections \ref{sec:gen_anal} and \ref{sec:prob_dep}, respectively.
Additionally, we provide a detailed comparison with existing algorithms in Section~\ref{sec:comp}.

\paragraph{Optimistic Direction Selection}

The key idea behind the ROFUL algorithm is that it uses an \emph{optimistic} action set ($\Yc_t^o$) to find which direction should be played to efficiently balance exploration and exploitation, while using a \emph{pessimistic} action set ($\Yc_t^p$) to find the scaling of this direction that will ensure constraint satisfaction.
In each round, the algorithm first finds the action $\xtil_t$ which maximizes the upper confidence bound over the optimistic set (line \ref{lne:opt_act}), and then finds the largest scalar $\gamma_t$ such that $\gamma_t \xtil_t$ is known to be in the pessimistic set (line \ref{lne:scale}).
The optimistic set overestimates the feasible set and the upper-confidence bound overestimates the reward, so $\xtil_t$ can be viewed as the optimistic action with respect to both the reward and the constraint.
As such, the algorithm uses $\xtil_t$ to determine which direction to play.
However, the action $\xtil_t$ may not satisfy the constraints, so it needs to be scaled down until it is within the pessimistic set and will therefore satisfy the constraints.

\paragraph{Confidence Sets for Unknown Parameters}

In order to construct the optimistic and pessimistic action sets as well as the upper confidence bound for the reward, we use confidence sets for the unknown parameters $\theta,a$.
To specify these confidence sets, we need to impose some structure on the noise terms.
In particular, the following assumption specifies that the noise terms $\epsilon_t, \eta_t$ are $\rho$-subgaussian conditioned on the history up to the point that $y_t,z_t$ are observed.

\begin{assumption}
  \label{ass:noise}
    For all $t \in [T]$, it holds that $\Eb[\epsilon_t | x_1, \epsilon_1, ..., \epsilon_{t-1}, x_t] = 0$ and $\Eb[\exp(\lambda \epsilon_t) | x_1, \epsilon_1, ...,  \epsilon_{t-1}, x_t] \leq \exp(\frac{\lambda^2 \rho^2}{2}), \forall \lambda \in \Rb$.
    The same holds replacing $\epsilon_t$ with $\eta_t$.
\end{assumption}

The specific confidence set that we use is from \citet{abbasi2011improved} and is given in the following.

\begin{lemma}[Theorem 2 in \citet{abbasi2011improved}]
  \label{lem:conf_sets}
  Let Assumptions \ref{ass:set_bound}, \ref{ass:bounded} and \ref{ass:noise} hold.
  Also, let
  \begin{equation}
    \label{eqn:beta}
    \beta_t := \rho \sqrt{d \log \left( \frac{1 + (t-1) / \lambda}{\delta / 2} \right) } + \sqrt{\lambda} S.
  \end{equation}
  Then with probability at least $1 - \delta$, it holds that both $| x^\top (\hat{\theta}_t - \theta)| \leq \beta_t \| x \|_{V_t^{-1}}$ and $| x^\top (\hat{a}_t - a)| \leq \beta_t \| x \|_{V_t^{-1}} $ for all $x \in \Rb^d$ and all $t \geq 1$.
\end{lemma}

It follows from Lemma \ref{lem:conf_sets} that, with high probability, the optimistic and pessimistic action sets contain and are contained by the true feasible set $\Yc$, respectively.
Since ROFUL only chooses actions from the pessimistic action set (or those with norm less than $\nu$), the actions chosen by the algorithm satisfy the constraints at all rounds with high probability.

\subsection{General Analysis}

\label{sec:gen_anal}

The ROFUL algorithm (Algorithm \ref{alg:main_alg}) is proven to enjoy sublinear regret and maintain constraint satisfaction in the following theorem.

\begin{theorem}
  \label{thm:main}
  Let Assumptions \ref{ass:set_bound}, \ref{ass:bounded} and \ref{ass:noise} hold.
  Then, with probability at least $1 - \delta$, the regret of ROFUL (Algorithm \ref{alg:main_alg}) satisfies
  \begin{equation}
  \label{eqn:reg_main}
      R_T \leq 2 \frac{\|\theta \| + S_a}{b} \beta_T \sqrt{2 d T \log\left(1 + \frac{T}{\lambda d} \right)},
  \end{equation}
  and $a^\top x_t \leq b$ for all $t \in [T]$.
\end{theorem}

Inspecting the bound in Theorem \ref{thm:main}, we can see that the regret is $\Oc\left( \frac{1}{b} d \sqrt{T} \log(T) \right)$, only considering $T$, $d$ and $b$.
This matches the orderwise regret of other safe upper-confidence bound approaches, as discussed in Section \ref{sec:comp}.
In the next section, we find that it is possible to achieve improved problem-dependent regret guarantees.

\begin{remark}
  The ROFUL algorithm and Theorem \ref{thm:main} easily extend to the setting where the action set $\Xc$ and constraint limit $b$ are allowed to vary in each round.
\end{remark}

\subsection{Problem-dependent Analysis}

\label{sec:prob_dep}

We also study the case where the optimal reward is well-separated from the reward of any feasible action that is not in the same direction as the optimal action.
To make this concrete, let the \emph{reward gap} be defined as\footnote{Unlike the problem-dependent analysis in \mbox{\cite{amani2019linear}}, our notion of a reward gap does not depend on how tight the constraints are on the optimal action.}
\begin{equation}
\label{eqn:delt}
  \Delta := \inf_{x \in \Yc:\ x \neq \alpha x_*\ \forall \alpha > 0 } \theta^\top (x_* - x),
\end{equation}
We study the case where $\Delta > 0$.
Note that the typical notion of a ``reward gap" in linear bandits, such as that used by the problem-dependent analysis in \cite{dani2008stochastic} and \cite{abbasi2011improved}, is not particularly useful in the safe linear bandit setting because it relies on the optimal reward being separated from the reward of any other action that the learner might play.
This could occur in the conventional linear bandit setting either when the feasible set is finite, which would not be a star convex set (except for the trivial case), or when the feasible set has finite extrema, which will not ensure that the played actions are well-separated in safe linear bandits given that the feasible set is unknown.
Nonetheless, when the constraint is loose (i.e. $\nu > 1$), a well-separated problem in our setting ($\Delta > 0$) implies a well-separated problem in the conventional linear bandit setting.

\paragraph{Wrong Directions are Rarely Selected}
We find that when $\Delta > 0$, we can establish a polylogarithmic bound on the number of times that ROFUL chooses the wrong direction, which is denoted by
\begin{equation*}
  B_T := \sum_{t=1}^{T} \indic\{ \nexists \ \alpha > 0 : x_t = \alpha x_* \}.
\end{equation*}
Specifically, the following theorem shows that ROFUL chooses $\Oc\left( \frac{1}{b^2 \Delta^2} d^2 \log^2(T) \right)$ wrong directions when ${\Delta > 0}$.
\begin{theorem}
  \label{thm:prob_dep}
  Let Assumptions \ref{ass:set_bound}, \ref{ass:bounded} and \ref{ass:noise} hold.
  If $\Delta > 0$, then the number of wrong directions chosen by ROFUL (Algorithm \ref{alg:main_alg}) satisfies
  \begin{equation*}
    B_T \leq \frac{32 S^2 \beta_T^2 d}{b^2 \Delta^2 }  \log\left(1 + \frac{T}{\lambda d} \right)
  \end{equation*}
  with probability at least $1 - \delta$.
\end{theorem}

\paragraph{Nearly Dimension-free Regret}

Leveraging Theorem \ref{thm:prob_dep}, we can devise a version of ROFUL that achieves improved regret guarantees when $\Delta > 0$ and known.
In particular, Theorem \ref{thm:prob_dep} implies that the optimal direction can be identified in a polylogarithmic number of rounds.
Once the optimal direction has been identified, the problem becomes one-dimensional and therefore does not suffer any dimensional dependence.

\begin{corollary}
  \label{thm:prob_dep_reg}
  Let Assumptions \ref{ass:set_bound}, \ref{ass:bounded} and \ref{ass:noise} hold.
  If $\Delta > 0$, consider the algorithm PD-ROFUL: \footnote{Detailed pseudo-code of PD-ROFUL is given in Algorithm \ref{alg:pd_oful} in Appendix \ref{apx:probdep}.}
  \begin{enumerate}
    \item Play ROFUL until any single direction has been played more than $\bar{B} := \frac{32 S^2 \beta_T^2 d}{b^2 \Delta^2 }  \log\left(1 + \frac{T}{\lambda d} \right)$ times. Let this direction be denoted by $u_*$.
    \item For the remaining rounds, play ROFUL (after restarting) for the 1-dimensional safe linear bandit problem of choosing $\xi_t \in \Rb_+$ and then playing~$\xi_t u_*$.
  \end{enumerate}
  Then, with probability at least $1 - 2 \delta$,
  \begin{equation*}
      R_T \leq \frac{4S}{b} \beta_{2 \bar{B} + 1} \sqrt{ 2 d (2 \bar{B} + 1)\log\left(1 + \frac{2 \bar{B} + 1}{\lambda d} \right)} + \frac{4 S}{b} \tilde{\beta}_{T} \sqrt{ 2 T\log\left(1 + \frac{T}{\lambda d} \right)}
  \end{equation*}
  where $\tilde{\beta}_{T}$ is $\beta_T$ with $d = 1$.
\end{corollary}

Corollary \ref{thm:prob_dep_reg} indicates that when $\Delta > 0$ and known, it is possible to achieve $\Octil (\frac{d^2}{b^2 \Delta} + \frac{1}{b}\sqrt{T})$ regret.
When $T$ is large and $\frac{1}{\Delta}$ is $\Oc(1)$, this improves on the general regret bound in Theorem \ref{thm:main} because the second term dominates.
Concretely, as $T$ goes to infinity,  $R_T/\sqrt{T}\log(T)$ goes to $\Oc(\frac{1}{b})$ whereas in the general case (i.e. Theorem \ref{thm:main}), it goes to $\Oc(d \frac{1}{b})$.
\begin{remark}
  This problem-dependent analysis approach yields similar guarantees for existing safe linear bandit algorithms as shown in Appendix \ref{apx:polb}.
\end{remark}

\subsection{Comparison with Existing Algorithms}
\label{sec:comp}

In this section, we discuss the key differences between ROFUL and existing safe linear bandit algorithms.
Compared to ROFUL, which uses an optimistic action set to identify low-regret actions, existing safe linear bandit algorithms often choose actions directly from the pessimistic action set using either linear UCB (\cite{pacchiano2021stochastic,amani2021decentralized}) or linear TS (\cite{moradipari2021safe}) where an expanded confidence set is used in both cases.
In our specific setting, the linear UCB version of the expanded confidence set approach can be written as
\begin{equation}
  \label{eqn:oplb}
  x_t \in \argmax_{x \in \Yc_t^p} \left( \hat{\theta}_t^\top x + \kappa \beta_t \| x \|_{V_t^{-1}} \right),
\end{equation}
with an appropriately chosen parameter $\kappa \geq 1$.
The specific choice of $\kappa$ ensures that optimism holds, i.e. that $\hat{\theta}_t^\top x_t + \kappa \beta_t \| x_t \|_{V_t^{-1}} \geq \theta^\top x_*$, which is critical to ensuring that the algorithm enjoys sublinear regret.
We call this generic algorithm GenOP (as in Generic Optimism-Pessimism) in reference to the concept of optimism-pessimism that is often used in safe linear bandits (e.g. \cite{pacchiano2021stochastic}).
Note that the choice of $\kappa$ used in existing UCB-based algorithms is not appropriate for our setting because such algorithms were developed for slightly different settings (i.e. decentralized \cite{amani2021decentralized}, local constraints \cite{varma2023stochastic}, or constraints in expectation \cite{pacchiano2021stochastic}) so we show in Appendix \ref{sec:gen_oplb} that it is sufficient to choose $\kappa = 1 + \frac{2 S_\theta}{b}$, to get 
\begin{equation}
  \label{eqn:oplb_reg}
  R_T \leq (1 + \kappa) \beta_{T} \sqrt{ 2 d T\log\left(1 + \frac{T}{\lambda d} \right)}.
\end{equation}

Note that because GenOP uses a \emph{fixed} $\kappa$ parameter that must be chosen ahead of time, it is necessarily defined using worst-case quantities (such as $S_\theta$).
Conversely, ROFUL uses the optimistic action set and safe scaling $\gamma_t$ which are updated with empirical quantities in each round and therefore improve as more data is collected.
This suggests that ROFUL is preferable in ``easier" problem instances in which worst-case quantities are loose on the true empirical quantities.

We can gain additional insight into the respective benefits of either algorithm by comparing the regret bounds.
Specifically, it follows from \eqref{eqn:reg_main} and \eqref{eqn:oplb_reg} that ROFUL enjoys a tighter regret bound than GenOP when
\begin{equation}
\label{eqn:comp}
  S_\theta - \| \theta \| > S_a - b.
\end{equation}
The quantity on the left-hand side represents how loose the assumed bound on the reward parameter is on the true value (given that Assumption \ref{ass:bounded} specifies that $\| \theta \| \leq S_\theta$), while the right-hand side represents how loose the assumed bound on the constraint limit ($b$) is (given that Assumption \ref{ass:bounded} specifies that $\nu = b/S_a \leq 1$ and therefore that $b \leq S_a$).
Therefore, \eqref{eqn:comp} suggests that ROFUL is preferred over GenOP when the bound on the reward parameter is loose and the bound on the constraint limit is tight.
Our numerical experiments support this intuition as ROFUL outperforms GenOP on average when $b$ is large (and therefore $S_a$ is tighter on $b$), while the two algorithms perform similarly in the settings when $b$ is small (and therefore $S_a$ is looser on $b$).

\section{Safe Phased Elimination Algorithm}
\label{sec:spe}

In this section, we propose the algorithm \emph{Safe Phased Elimination} (Safe-PE) for the case when the action set is a finite star-convex set.
We provide a high-level description of Safe-PE here and give the full algorithm in Appendix \ref{apx:safepe}.
The assumption that the action set is a finite star-convex set means that it can be represented as
\begin{equation}
  \label{eqn:fin_star}
  \Xc = \bigcup_{i \in [k]} \{\alpha u_i : \alpha \in [0,\alpha_i] \},
\end{equation}
where $u_1, ..., u_k \in \Sb$ are unit vectors and $\alpha_1, ..., \alpha_k \in \Rb_{++}$ are the maximum scalings for each unit vector.
We find that in such a setting, it is possible to reduce the dependence on the problem dimension when $k \ll 2^d$.
The key insight is that a confidence set at a single action applies to all scalings of that action without the need for a union bound over a cover (or related technique).
This insight allows us to leverage the reduced dimension dependence offered by phased elimination algorithms in the safe linear bandit setting.
Nonetheless, it also introduces additional challenges due to the fact that the pessimistic action set varies from phase to phase.

\paragraph{Algorithm Description}
Our Safe-PE algorithm operates in phases $j = 1,2,...$ that grow exponentially in duration, and maintains a set of viable directions $\Ac$ and a pessimistic set of actions $\Yc^p$ that are updated in each round.
In particular, each phase $j$ proceeds as:
\begin{enumerate}
  \item For $2^{j-1}$ rounds, play the action with the largest confidence set width $\| \cdot \|_{V_t^{-1}}$ in each round.
  \item Eliminate directions from $\Ac$ that have low estimated reward.
  \item Update $\Yc^p$ by scaling the directions in $\Ac$ as large as possible while still being verifiably safe.
\end{enumerate}
This algorithm builds on existing phased elimination algorithms, including those from \citet{auer2002using}, \citet{chu2011contextual} and, specifically, \citet{valko2014spectral} and \citet{kocak2020spectral}.
However, Safe-PE differs in that it eliminates directions, instead of distinct actions, and maintains a set of safe actions to ensure constraint satisfaction.
Furthermore, it requires a looser criterion when eliminating directions to ensure that the optimal direction is not eliminated.

\paragraph{Regret Analysis}

As is commonly used for phased elimination algorithms \cite{auer2002using,chu2011contextual,lattimore2020learning}, we assume that the noise terms are independent subgaussian random variables.
\begin{assumption}
  \label{ass:pe_noise}
  The noise sequences $(\epsilon_t)_{t=1}^T$ and $(\eta_t)_{t=1}^T$ are sequences of independent $\rho$-subgaussian random variables.
\end{assumption}
With this, we state the regret guarantees for the Safe-PE algorithm in Theorem \ref{thm:safepe}.
\begin{theorem}
  \label{thm:safepe}
  Let Assumptions \ref{ass:set_bound}, \ref{ass:bounded} and \ref{ass:pe_noise} hold.
  When the action set is a finite star-convex set, the regret of Safe-PE (Algorithm \ref{alg:pe1} in Appendix \ref{apx:safepe}) is $\Octil(\frac{1}{b^2}\sqrt{d T})$.
\end{theorem}

Theorem \ref{thm:safepe} shows that, for the case when the action set is a finite star convex set, the regret of Safe-PE is $\Octil(\sqrt{d T})$ in terms of $d$ and $T$.
Note that the regret only depends on the number of directions ($k$) in log factors and therefore this improves on the regret of ROFUL in terms of $d$ when $k \ll 2^d$.
For example, if the directions are the coordinate directions, i.e. $u_i = e_i$, then $k = 2d$ and therefore the regret bound of Safe-PE is $\Octil(\sqrt{d T})$ since $d$ only appears in log factors.
However, if the directions are the corners of the hypercube, then $k = 2^d$ and the regret bound of Safe-PE is $\Octil(d \sqrt{T})$.
Also, note that the regret bound of Safe-PE depends on $\frac{1}{b^2}$, whereas the regret bound of ROFUL depends on $\frac{1}{b}$.
As such, the regret bound of ROFUL is still tighter than that of Safe-PE in some settings, e.g. when $b$ is small and $d = 1$.

\section{Extension to Linked Convex Constraints}
\label{sec:conv}

\begin{figure}[t]
  \centering{
  \includegraphics[width=0.5\columnwidth]{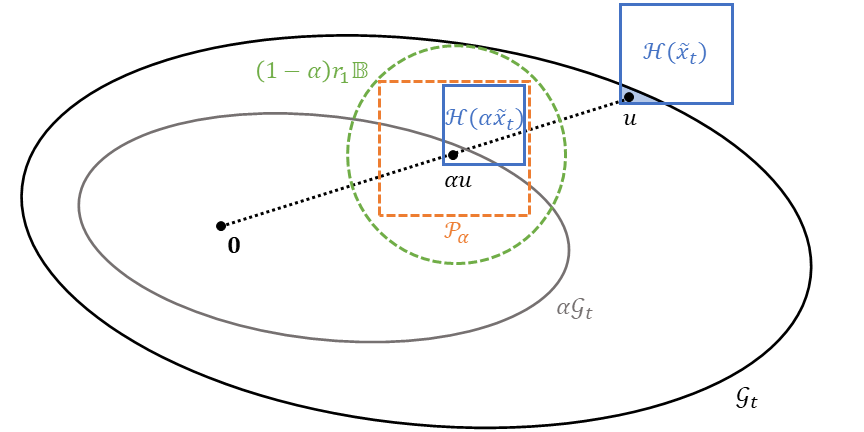}
  }
  \caption{Graphical representation of the approach for lower bounding $\gamma_t$ (in ROFUL) for the setting with linked convex constraints.}
  \label{fig:pf_vis}
  % \vspace{-0.15in}
\end{figure}

In this section, we generalize the design and analysis of the algorithms ROFUL, Safe-PE and GenOP to a novel setting which we call \emph{linked convex constraints}, where the output of the constraint function is multi-dimensional and must lie in an arbitrary convex set.
The key challenge in this setting is characterizing how far a point in the optimistic action set is from the pessimistic action set.
To address this, we leverage a theoretical tool from the zero-order optimization literature.
We only provide a description of key contributions in this section and leave the details of the algorithms to Appendix~\ref{apx:convex}.

\paragraph{Problem Description}
The problem setting is specified as follows.
At each round $t \in [T]$, the learner observes $z_t = A x_t + \eta_t$, where $A \in \Rb^{n \times d}$ is the unknown constraint matrix and $\eta_t \in \Rb^n$ is a vector noise term.
The learner must ensure that $A x_t$ is in the known convex set $\Gc$ for all $t \in [T]$.
The reward function and feedback mechanism are the same as the original setting described in Section \ref{sec:pset}.
We assume that there exists $r > 0$ such that $r \Bb \subseteq \Gc$.
Lastly, we assume that each element of $\eta_t$ satisfies the assumptions on the noise used for ROFUL (Assumption \ref{ass:noise}) or Safe-PE (Assumption \ref{ass:pe_noise}).

\paragraph{Analysis Techniques}
Although the design of the algorithms trivially extends to this setting, the regret analysis is more challenging.
In particular, it is difficult to characterize the distance from any point in the optimistic action set to the pessimistic action set.
To address this, we use an analysis tool that is popular in the zero-order optimization literature (e.g. \citet{flaxman2005online}).
This tool, given in Fact \ref{prop:shrunk}, allows us to consider a shrunk version of the constraint set in order to bound the scaling that is required to take any point in $\Yc_t^o$ to $\Yc_t^p$, i.e. $\gamma_t$ in ROFUL.
We use a similar approach to bound the scaling required to take any point in $\Yc$ to $\Yc_t^p$ for GenOP and Safe-PE.

\begin{Fact}
  \label{prop:shrunk}
  Let $\Dc$ be a convex set such that $r \Bb \subseteq \Dc$ for some $r \geq 0$.
  Then, for any $\alpha \in [0,1]$ and $x \in \alpha \Dc$, it holds that $x + (1 - \alpha) r \Bb \subseteq \Dc$.
\end{Fact}

With this in hand, we can then describe our approach for lower bounding $\gamma_t$, which is illustrated in Figure~\ref{fig:pf_vis}.
Recall the definition of $\xtil_t$ in ROFUL, where it is known that $\xtil_t$ is in $\Yc_t^o$.
Then, the overall objective is to find some positive scaling $\alpha$ such that $\alpha \xtil_t$ is in $\Yc_t^p$ and then it follows that $\gamma_t \geq \alpha$.
To do so, we define the uncertainty set for the constraint function at point $x$ as the box $\Hc (x) := \hat{a}_t^\top x + \beta_t \| x \|_{V_t^{-1}} \Bb_\infty$ and note that $\Yc_t^p$ and $\Yc_t^o$ are precisely the set of $x \in \Xc$ such that $\Hc(x)$ has nonempty intersection with $\Gc$ and the set of $x \in \Xc$ such that $\Hc(x)$ is contained in $\Gc$, respectively.
First, we consider a point $u$ in the intersection of $\Gc$ and $\Hc (\xtil_t)$.
Such a point exists given that $\xtil_t$ is in $\Yc_t^o$.
Next, we scale $u$ by some non-negative scalar $\alpha$.
Note that $\alpha u$ is in $\Hc(\alpha \xtil_t)$ given that $\Hc$ is positive homogeneous, i.e. $\alpha \Hc(x) = \Hc(\alpha x)$ for any $x$.
In order to show that $\alpha \xtil_t$ is in $\Yc_t^p$, we need to show that $\Hc(\alpha \xtil_t)$ is contained in $\Gc$.
To do so, we first consider a set $\Pc_\alpha$ that is centered at $\alpha u$ but has twice the radius of $\Hc(\alpha \xtil_t)$ and therefore contains $\Hc(\alpha \xtil_t)$ (this is illustrated in Figure \ref{fig:pf_vis}).
We then use Fact \ref{prop:shrunk} to reason that, because $u$ is in $\Gc$, the ball $\alpha u + (1 - \alpha) r_1 \Bb$ is contained in $\Gc$.
Therefore, we choose $\alpha$ such that $(1 - \alpha) r_1 \Bb = 2 \alpha \sqrt{n} \beta_t \| \xtil_t \|_{V_t^{-1}}$, where the $\sqrt{n}$ is necessary to bound an infinity-norm ball with a 2-norm ball.
Some simple algebra shows that $\gamma_t \geq 1 - \frac{2 \sqrt{n}}{r} \beta_t \| x_t \|_{V_t^{-1}}$.

\section{Numerical Experiments}
\label{sec:num_lin}

\begin{figure*}[t]
  \begin{center}
  % \hfill
  % \hspace{0.2in}
  \begin{subfigure}[t]{0.24\textwidth}
    \centering\includegraphics[width=\textwidth]{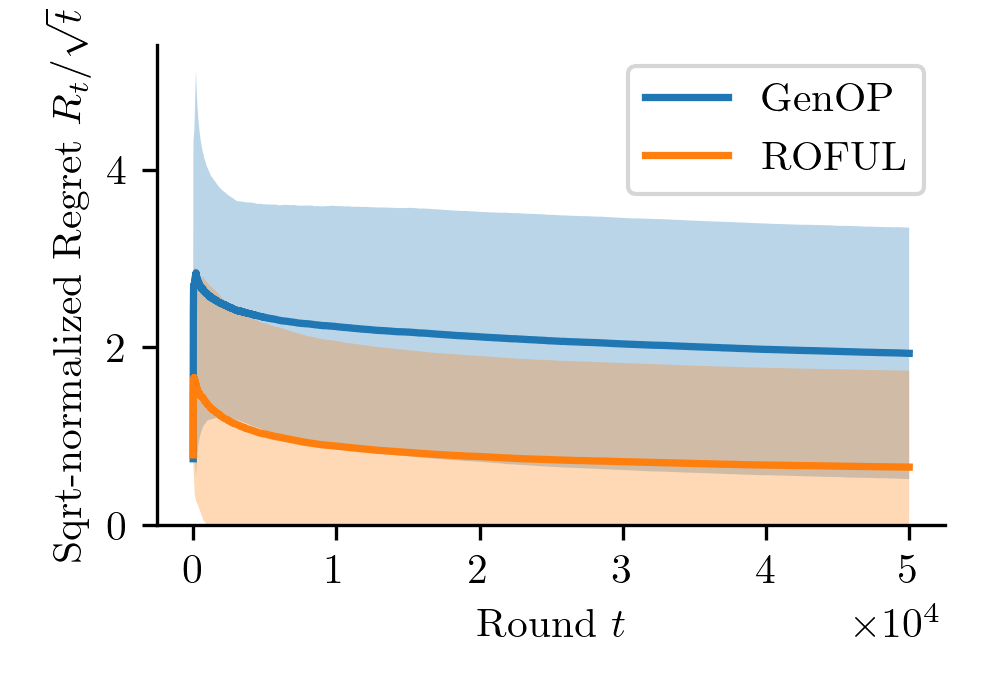}
    \caption{Linear constraints with large $b$.}
    \label{fig:sims:a}
  \end{subfigure}
  \hfill
  \begin{subfigure}[t]{0.24\textwidth}
    \centering\includegraphics[width=\textwidth]{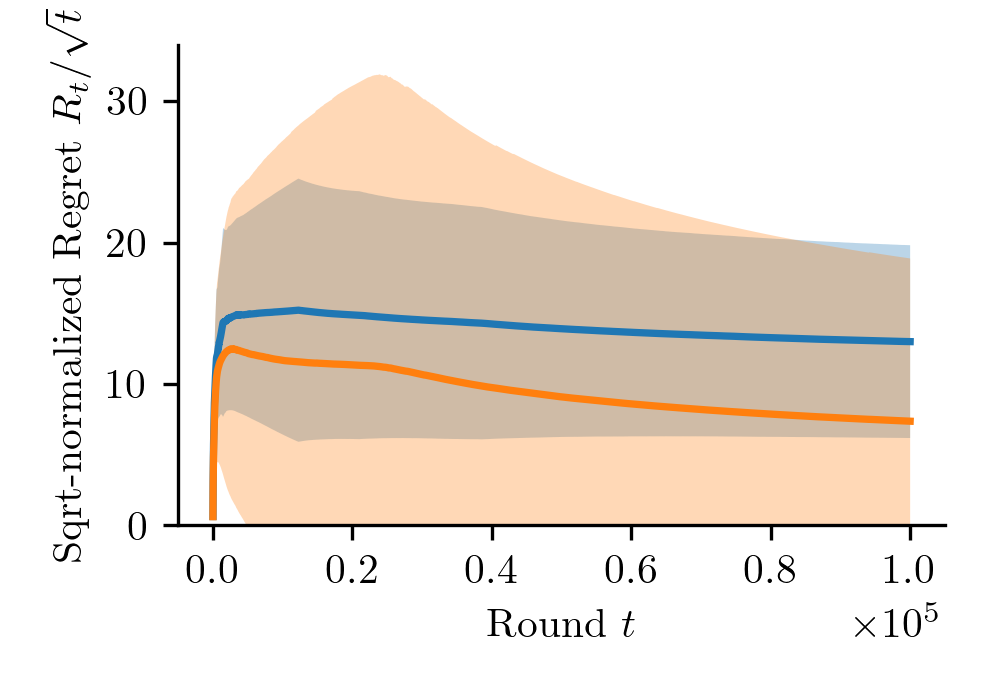}
    \caption{Linear constraints with small $b$.}
    \label{fig:sims:b}
  \end{subfigure}
  \hfill
  \begin{subfigure}[t]{0.24\textwidth}
    \centering\includegraphics[width=\textwidth]{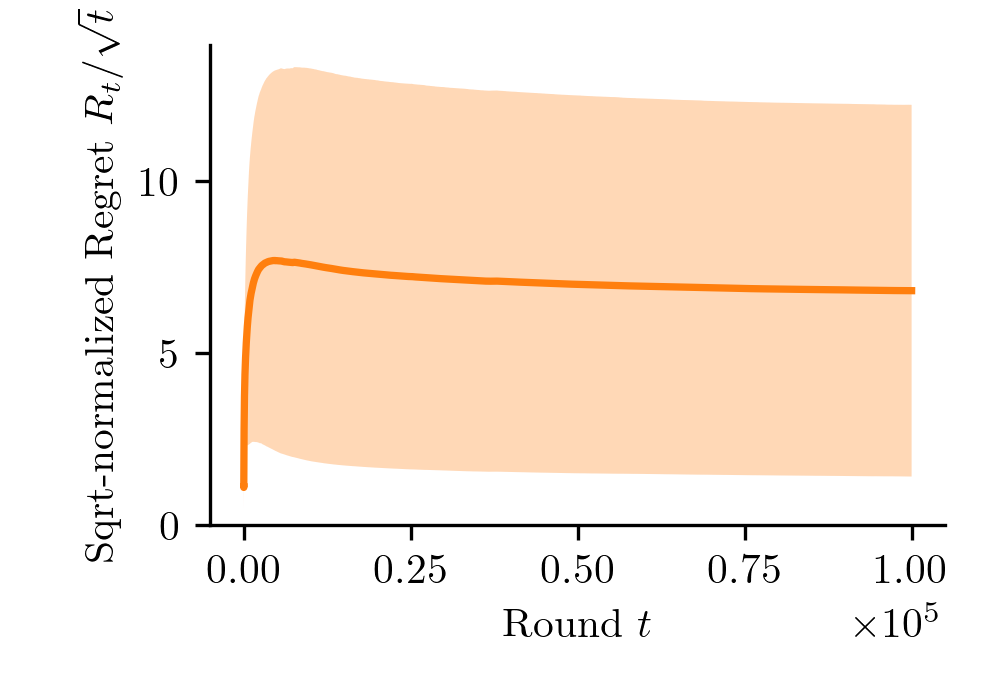}
    \caption{Linked convex constraints.}
    \label{fig:sims:c}
  \end{subfigure}
  \hfill
  \begin{subfigure}[t]{0.24\textwidth}
    \centering\includegraphics[width=\textwidth]{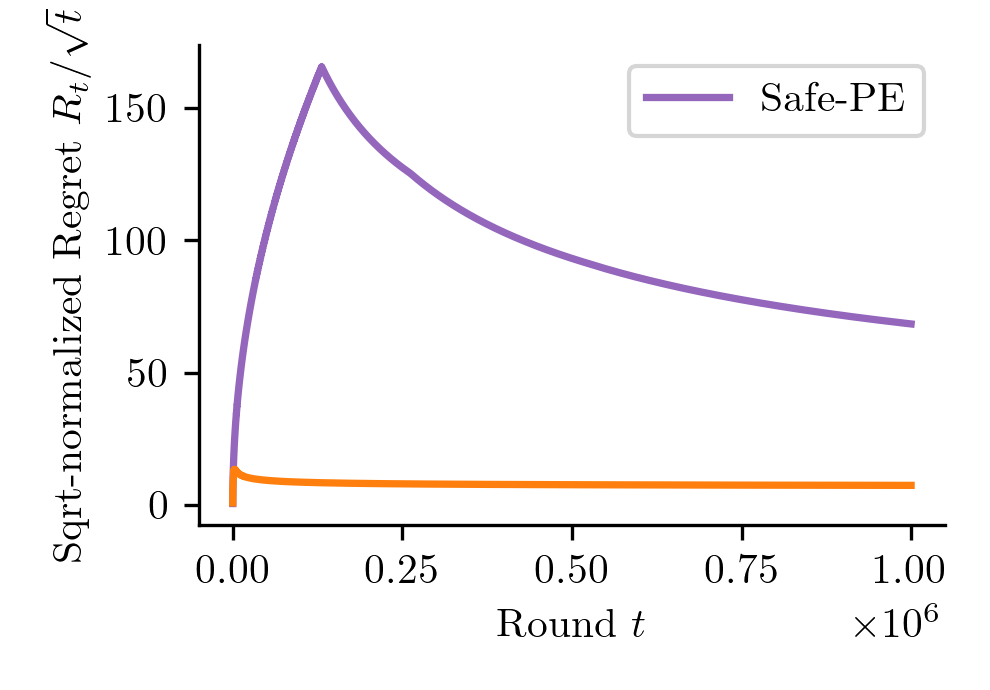}
    \caption{Star-convex multi-armed bandit.}
    \label{fig:sims:d}
  \end{subfigure}
  \end{center}
  \vspace{-0.1in}
  \caption{Simulation results of our proposed algorithms (ROFUL, Safe-PE) and generic expanded confidence set algorithm GenOP (see Section \ref{sec:comp}).}
  \label{fig:sims}
\end{figure*}

In this section, we numerically validate the theoretical guarantees and assess the performance of the proposed algorithms.
Note that we only give a high-level description of the simulations in this section.
The details of the experimental settings and additional results are given in Appendix~\ref{apx:num_exp}.

First, we consider a setting with a linear constraint.
We study the case when $b$ is large (Figure \ref{fig:sims:a}) and when $b$ is small (Figure \ref{fig:sims:b}).\footnote{The code used to generate these results is available at \url{https://github.com/shutch1/Directional-Optimism-for-Safe-Linear-Bandits}.}
We simulate ROFUL and GenOP for $30$ trials for each case, where $b$ is uniformly sampled in the interval $[0.25,1]$ for the first case and in the interval $[0.05,0.25]$ for the second case.
The action set is taken to be a finite star-convex set with $10$ directions that are randomly sampled in each trial.
Furthermore, the reward vector $\theta$, constraint vector $a$, constraint limit $b$ and noise realizations $\epsilon_t, \eta_t$ are also randomly sampled in each trial.
The average and standard deviation of the regret normalized by $\sqrt{t}$ are shown in Figures \ref{fig:sims:a} and \ref{fig:sims:b}.
These plots suggest that, when $b$ is large, ROFUL outperforms GenOP in the aggregate.
When $b$ is small, the average performance of the two algorithms is similar, although GenOP enjoys a smaller standard deviation than ROFUL.

Next, we consider a setting with convex constraints.
In particular, we study the case where the constraint set is a ball, i.e. $\Gc = b\Bb$ for scalar $b$.
We simulated ROFUL for $30$ trials in each setting, where constraint set radius $b$, the reward vector $\theta$, constraint vector $a$ and noise realizations $\epsilon_t, \eta_t$ are all randomly sampled.
The average and standard deviation of the regret normalized by $\sqrt{t}$ is shown in Figure \ref{fig:sims:c}.
In this plot, ROFUL converges to constant $\sqrt{t}$ regret.
We provide additional results for the case when $\Gc$ is an infinity-norm ball in Figure \ref{fig:box} in Appendix \ref{apx:num_exp}.

Lastly, we consider a star-convex multi-armed bandit with results shown in Figure \ref{fig:sims:d}.
In this setting, the action set only includes the coordinate directions.
We simulate both ROFUL and Safe-PE in this setting with $d=10$.
The regret normalized by $\sqrt{t}$ is shown in Figure \ref{fig:sims:d}.
From this plot, it is clear that ROFUL outperforms Safe-PE despite the fact that Safe-PE enjoys a tighter regret bound in terms of the problem dimension.
In fact, it is well known that UCB-based algorithms often empirically outperform elimination-based algorithms despite the orderwise tighter regret bound, as discussed by \citet{valko2014spectral} and \citet{chu2011contextual}.
Simulation results for PD-ROFUL (specified in Corollary \ref{thm:prob_dep_reg}) are given in Appendix \ref{apx:num_exp}.

\section{Conclusion}

In this work, we take a novel approach to the safe linear bandit problem in which we view it as fundamentally a problem of choosing directions rather than actions.
We find that this approach leads to improvements in empirical performance in certain problem instances as well as tighter geometry-dependent regret bounds.
An interesting direction for future work is to investigate if this approach yields similar gains when applied to related safe learning problems, such as constrained MDPs or safe Gaussian process optimization.

\subsubsection*{Acknowledgements}

This work was supported by NSF grant \#1847096.

\bibliographystyle{plainnat}
\bibliography{references}

\onecolumn

\startcontents[appendices]
\printcontents[appendices]{}{1}{\section*{Contents}}

\appendix

\section{Proof of Theorem \ref{thm:main}}

In this section, we prove the general regret bound given in Theorem \ref{thm:main}.
First, we introduce some notation.
Let the event that the confidence sets hold be defined as
\begin{equation}
  \label{eqn:econf}
    \Econf := \{ | x^\top (\hat{\theta}_t - \theta) | \leq \beta_t \| x \|_{V_t^{-1}}, | x^\top (\hat{a}_t - a) | \leq \beta_t \| x \|_{V_t^{-1}}, \forall x \in \Rb^d, \forall t \geq 1 \},
\end{equation}
and note that $\Pb(\Econf) \geq 1 - \delta$ by Lemma \ref{lem:conf_sets}.

We start by giving a key lemma that lower bounds $\gamma_t$.
\begin{lemma}
  \label{lem:gam2}
  If Assumptions \ref{ass:set_bound} and \ref{ass:bounded}, and $\Econf$ holds, then
  \begin{equation*}
      \gamma_t \geq \max \left( 1 - \frac{2}{b} \beta_t \| x_t \|_{V_t^{-1}}, \nu \right)
  \end{equation*}
  for all $t \in [T]$.
\end{lemma}
\begin{proof}
  We will find lower bounds individually for $\mu_t$ (line \ref{lne:scale}) and $\btil_t$ (line \ref{lne:pess_scale}) in the following.

  \textbf{Lower bound on $\mu_t$:}
  Since $\mu_t = \max \left\{\mu \in [0,1]: \mu \xtil_t \in \Yc_t^p \right\}$, we find a lower bound on $\mu_t$ by finding an $\alpha \in \left\{\mu \in [0,1]: \mu \xtil_t \in \Yc_t^p \right\}$.
  Specifically, we will show that $\alpha$ can be chosen as
  \begin{equation*}
      \alpha = \frac{b}{b +  2 \beta_t \| \xtil_t \|_{V_t^{-1}}}.
  \end{equation*}
  For this to be a valid choice for $\alpha$, we need that (i) $\alpha \in [0,1]$, (ii) $\alpha \xtil_t \in \Xc$ and (iii) $\hat{a}_t^\top (\alpha \xtil_t) + \beta_t \| \alpha \xtil_t \|_{V_t^{-1}} \leq b$.
  Point (i) follows by definition.
  Point (ii) holds because $\xtil_t \in \Xc$, $\Xc$ is star-convex and $\alpha \in [0,1]$, so $\alpha \xtil_t \in \Xc$.
  Then, to show point (iii), we have that
  \begin{align*}
      \hat{a}_t^\top (\alpha \xtil_t) + \beta_t \| \alpha \xtil_t \|_{V_t^{-1}} & = \alpha (\hat{a}_t^\top \xtil_t + \beta_t \| \xtil_t \|_{V_t^{-1}} )\\
      & = \alpha (\hat{a}_t^\top \xtil_t - \beta_t \| \xtil_t \|_{V_t^{-1}} + 2 \beta_t \| \xtil_t \|_{V_t^{-1}} )\\
      & \leq \alpha (b + 2 \beta_t \| \xtil_t \|_{V_t^{-1}} )\\
      & = b,
  \end{align*}
  where the inequality uses the fact that $\xtil_t \in \Yc_t^o$ and therefore $\hat{a}_t^\top \xtil_t - \beta_t \| \xtil_t \|_{V_t^{-1}} \leq b$, and the last equality uses the choice of $\alpha$.
  Therefore, 
  \begin{equation*}
      \mu_t \geq \alpha = \frac{b}{b +  2 \beta_t \| \xtil_t \|_{V_t^{-1}}}.
  \end{equation*}
  
  \textbf{Lower bound on $\tilde{b}_t$:}
  Recall that,
  \begin{equation*}
      \btil_t = \begin{cases} \min\left(  \frac{ \nu}{\|\xtil_t\|} ,1 \right) & \mathrm{if} \ \xtil_t \neq \mathbf{0},\\
    1 & \mathrm{else.}
    \end{cases}
  \end{equation*}
  We consider each case separately.
  If $\xtil_t = \mathbf{0}$, then $\btil_t = 1 \geq \nu$ given that $\nu \leq 1$ by Assumption \ref{ass:bounded}.
  Alternatively, if $\xtil_t \neq \mathbf{0}$, then
  \begin{equation}
      \label{eqn:btil_bound}
      \btil_t = \min\left(  \frac{ \nu}{ \|\xtil_t\|} ,1 \right) \geq \min\left(  \nu ,1 \right) = \nu,
  \end{equation}
  where the inequality holds because $\xtil_t \in \Xc$, and therefore, $\| \xtil_t \| \leq 1$ by Assumption \ref{ass:set_bound} which implies that $\frac{ \nu}{ \|\xtil_t\|} \geq \nu$.
  The last equality holds because $\nu \leq 1$ by Assumption \ref{ass:bounded}.
  Therefore, it holds that $\btil_t \geq \nu$ in either case.

  \textbf{Completing the proof:}
  With the above, we have shown that
  \begin{equation}
  \label{eqn:gam_bound}
      \gamma_t = \max \left( \btil_t , \mu_t \right) \geq \max \left( \nu , \frac{b}{b +  2 \beta_t \| \xtil_t \|_{V_t^{-1}}} \right)
  \end{equation}
  In order to complete the proof, we need a bound on $\gamma_t$ that is in terms of $\| x_t \|_{V_t^{-1}}$ instead of $\| \xtil_t \|_{V_t^{-1}}$.
  To get this, first note that $\gamma_t \xtil_t = x_t$ and $\gamma_t \geq 0$, and therefore
  \begin{equation*}
    \gamma_t \| \xtil_t \|_{V_t^{-1}} = \| \gamma_t   \xtil_t \|_{V_t^{-1}} = \| x_t \|_{V_t^{-1}}.
  \end{equation*}
  Using this, we can rearrange \eqref{eqn:gam_bound} to get that
  \begin{equation}
    \label{eqn:gam_eq}
    \begin{split}
      \gamma_t \geq \frac{b}{b +  2 \beta_t \| \xtil_t \|_{V_t^{-1}}} \quad \Longleftrightarrow \quad \gamma_t b + 2 \beta_t \| x_t \|_{V_t^{-1}} \geq b \quad \Longleftrightarrow \quad \gamma_t \geq 1 -  \frac{2}{b} \beta_t \| x_t \|_{V_t^{-1}}.
    \end{split}
  \end{equation}
  Finally, combining \eqref{eqn:gam_bound} and \eqref{eqn:gam_eq}, we get that
  \begin{equation*}
    \gamma_t \geq \max \left( 1 -  \frac{2}{b} \beta_t \| x_t \|_{V_t^{-1}} , \nu  \right),
  \end{equation*}
  completing the proof.
\end{proof}

Then, we turn our attention to the instantaneous regret.
In particular, we will utilize the decomposition,
\begin{equation*}
    r_t := \theta^\top (x_* - x_t) = \underbrace{\theta^\top (x_* - \xtil_t)}_{\tone} + \underbrace{\theta^\top (\xtil_t - x_t).}_{\ttwo}
\end{equation*}
$\tone$ can be understood as the regret due to the optimistic actions, while $\ttwo$ can be understood as the cost incurred by maintaining constraint satisfaction.
In the following two lemmas, we bound $\tone$ and $\ttwo$ seperately.

\begin{lemma}
  \label{lem:opt_reg}
  Conditioned on $\Econf$, it holds that
  \begin{equation*}
      \tone = \theta^\top \left( x_* - \xtil_t \right) \leq \frac{2}{\nu} \beta_{t} \| x_t \|_{V_{t}^{-1}}.
  \end{equation*}
\end{lemma}
\begin{proof}
  We condition on $\Econf$ without further reference.
  First, it holds for all $t \in [T]$ that
  \begin{equation}
      \label{eqn:optim}
      \theta^\top x_* = \max_{x \in \Yc} \theta^\top x \leq \max_{x \in \Yc} \left( \hat{\theta}_{t}^\top x + \beta_{t} \| x  \|_{V_{t}^{-1}} \right) \leq \max_{x \in \Yc_t^o} \left( \hat{\theta}_{t}^\top x + \beta_{t} \| x  \|_{V_{t}^{-1}} \right) = \hat{\theta}_{t}^\top \xtil_t + \beta_{t} \| \xtil_t \|_{V_{t}^{-1}}.
  \end{equation}
  Also, note that $\gamma_t \geq \nu > 0$ by Lemma \ref{lem:gam2} and therefore
  \begin{equation}
  \label{eqn:xt_scale}
      \| \xtil_t\|_{V_{t}^{-1}} = \left\| \frac{x_t}{\gamma_t}  \right\|_{V_{t}^{-1}} = \frac{1}{\gamma_t} \| x_t  \|_{V_{t}^{-1}} \leq \frac{1}{\nu} \| x_t  \|_{V_{t}^{-1}}.
  \end{equation}
  Therefore, it holds that
  \begin{align*}
      \tone = \theta^\top  x_* -  \theta^\top \xtil_t  & \leq \hat{\theta}_{t}^\top \xtil_t + \beta_{t} \| \xtil_t \|_{V_{t}^{-1}} -  \theta^\top \xtil_t \\
      & = (\hat{\theta}_{t}^\top - \theta)^\top \xtil_t + \beta_{t} \| \xtil_t \|_{V_{t}^{-1}} \\
      & \leq 2 \beta_{t} \| \xtil_t \|_{V_{t}^{-1}} \\
      & \leq \frac{ 2}{\nu} \beta_{t} \| x_t \|_{V_{t}^{-1}},
  \end{align*}
  where the first inequality uses \eqref{eqn:optim}, the second inequality uses the definition of $\Econf$ and the third inequality uses \eqref{eqn:xt_scale}.
\end{proof}

\begin{lemma}
  \label{lem:reg_safe}
  Let Assumption \ref{ass:set_bound} hold. Then, conditioned on $\Econf$, it holds that
  \begin{equation*}
      \ttwo = \theta^\top (\xtil_t - x_t) \leq \| \theta \| \frac{2 \beta_t }{b}  \| x_t \|_{V_t^{-1}} 
  \end{equation*}
  for all $t \in [T]$.
\end{lemma}
\begin{proof}
  Conditioned on $\Econf$, it holds that
  \begin{align*}
    \ttwo = \theta^\top (\xtil_t - x_t) & = \theta^\top (\xtil_t -  \gamma_t \xtil_t)\\
      & = (1 - \gamma_t) \theta^\top \xtil_t \\
      & \leq (1 - \gamma_t) \| \theta \| \| \xtil_t \| \\
      & \leq (1 - \gamma_t) \| \theta \| \\
      & \leq \| \theta \| \frac{2 \beta_t }{b}  \| x_t \|_{V_t^{-1}},
  \end{align*}
  where the second inequality uses the fact that $\xtil_t \in \Xc$ which implies that $\| \xtil_t \| \leq 1$ by Assumption \ref{ass:set_bound}, and the third inequality uses Lemma \ref{lem:gam2}.
\end{proof}

Finally, we turn our attention to the cumulative regret.
To bound this, we will need the so-called elliptic potential, which is standard in the stochastic linear bandit literature.

\begin{lemma}[Lemma 11 in \cite{abbasi2011improved}]
  \label{lem:sp_elip}
  Consider a sequence $(w_k)_{k \in \Nb}$ where $w_k \in \Rb^d$ and $\|w_k\| \leq 1$ for all $k \in \Nb$.
    Let $W_k = \lambda I + \sum_{s=1}^{k-1} w_s w_s^\top$ for some $\lambda \geq 1$.
    Then, it holds that
  \begin{equation*}
    \sum_{k=1}^K \| w_{k} \|_{W_{k}^{-1}}^2 \leq 2 d \log\left(1 + \frac{K }{\lambda d} \right).
  \end{equation*}
\end{lemma}

With this, we complete the proof of Theorem \ref{thm:main} in the following.

\begin{theorem}[Duplicate of Theorem \ref{thm:main}]
  \label{thm:reg_bound}
  Let Assumptions \ref{ass:set_bound}, \ref{ass:bounded} and \ref{ass:noise} hold.
  Then, with probability at least $1 - \delta$, the regret of ROFUL (Algorithm \ref{alg:main_alg}) satisfies
  \begin{equation*}
      R_T \leq 2 \left(\frac{\| \theta \|}{b} +  \frac{1}{\nu} \right) \beta_T \sqrt{2 d T \log\left(1 + \frac{T}{\lambda d} \right)},
  \end{equation*}
  and $a^\top x_t \leq b$ for all $t \in [T]$.
\end{theorem}
\begin{proof}
  We condition on $\Econf$ without further reference.
  We give the regret guarantee and then the safety guarantee in the following.

  \textbf{Regret guarantee:}
  Using Lemmas \ref{lem:opt_reg} and \ref{lem:reg_safe}, it holds that
  \begin{align*}
      r_t & = \theta^\top \left( x_* - x_t \right)\\
      & = \theta^\top \left( x_* - \xtil_t \right) + \theta^\top \left( \xtil_t - x_t \right)\\
      & = \tone + \ttwo\\
      & \leq 2 \left(\frac{\| \theta \|}{b} +  \frac{1}{\nu} \right) \beta_t \| x_t \|_{V_{t}^{-1}}.
  \end{align*}
  We can then study the sum of the squared instantaneous regret,
  \begin{align*}
    \sum_{t=1}^{T} r_t^2 & \leq \sum_{t=1}^{T} 4 \left(\frac{\| \theta \|}{b} +  \frac{1}{\nu} \right)^2  \beta_t^2 \| x_t \|_{V_{t}^{-1}}^2 \\
    & = 4 \left(\frac{\| \theta \|}{b} +  \frac{1}{\nu} \right)^2 \sum_{t=1}^{T}  \beta_t^2 \| x_t \|_{V_{t}^{-1}}^2 \\
    & \leq 4 \left(\frac{\| \theta \|}{b} +  \frac{1}{\nu} \right)^2 \beta_T^2 \sum_{t=1}^{T} \| x_t \|_{V_{t}^{-1}}^2 \\
    & \leq 8 \left(\frac{\| \theta \|}{b} +  \frac{1}{\nu} \right)^2 d \beta_T^2 \log\left(1 + \frac{T}{\lambda d} \right),
  \end{align*}
  where the second inequality uses the fact that $\beta_t$ is monotone in $t$ and the third inequality uses Lemma \ref{lem:sp_elip}.
  Then, by Cauchy-Schwarz, it holds that
  \begin{align*}
    R_T & = \sum_{t=1}^{T} r_t \\
    & \leq \sqrt{T \sum_{t=1}^{T} r_t^2}\\
    & \leq \sqrt{8 T \left(\frac{\| \theta \|}{b} +  \frac{1}{\nu} \right)^2 d \beta_T^2 \log\left(1 + \frac{T}{\lambda d} \right)}\\
    & = 2 \left(\frac{\| \theta \|}{b} +  \frac{1}{\nu} \right) \beta_T \sqrt{2 d T \log\left(1 + \frac{T}{\lambda d} \right)}.
  \end{align*}

  \textbf{Safety guarantee:}
  In order to show that $a^\top x_t \leq b$, we note that $\gamma_t = \max(\btil_t,\mu_t)$ and therefore it holds that either $\gamma_t = \btil_t$ or $\gamma_t = \mu_t$.
  If $\gamma_t = \btil_t$ and $\xtil_t \neq \mathbf{0}$, then using the quantity $\nu = b/S_a$ as defined in Assumption \ref{ass:bounded}, it holds that
  \begin{equation*}
      a^\top x_t \leq \| a \| \| x_t \| \leq S_a \| x_t \| = S_a \| \btil_t \xtil_t \| = S_a \btil_t \| \xtil_t \| = S_a \min\left( \frac{\nu}{\| \xtil_t \|}, 1 \right) \| \xtil_t \| \leq S_a \frac{\nu}{\| \xtil_t \|} \| \xtil_t \| = S_a \nu = b,
  \end{equation*}
  where we use Assumption \ref{ass:bounded} in the second inequality.
  If $\xtil_t = \mathbf{0}$, then $\btil_t = 1$ and therefore $x_t = \xtil_t = \mathbf{0}$ which implies that $a^\top x_t = 0 < b$.
  
  Alternatively, if $\gamma_t = \mu_t$, it holds that
  \begin{equation*}
      x_t = \mu_t \xtil_t \in \Yc_t^p \subseteq \Yc.
  \end{equation*}
  Therefore, it holds for both cases that $a^\top x_t \leq b$ for all $t \in [T]$.
\end{proof}

\section{Proof of Theorem \ref{thm:prob_dep}}
\label{apx:probdep}

In this section, we prove the problem-dependent analysis given in Theorem \ref{thm:prob_dep} and then Corollary \ref{thm:prob_dep_reg}.
In order to do so, we first restate the definition of $\Delta$ as follows
\begin{equation*}
  \Delta := \inf_{x \in \Yc:\ x \neq \alpha x_*\ \forall \alpha > 0 } \theta^\top (x_* - x),
\end{equation*}
and then restate the definition of $B_T$,
\begin{equation*}
  B_T := \sum_{t=1}^{T} \indic\{ \nexists \ \alpha > 0 : x_t = \alpha x_* \}.
\end{equation*}

Then, we give a lemma with some useful facts.

\begin{lemma}
\label{lem:vt}
    Let Assumptions \ref{ass:set_bound} and \ref{ass:bounded} hold, and let $\Econf$ hold.
    Also, let
    \begin{equation}
    \label{eqn:a_bound}
        \zeta_t := \max\{ \zeta \geq 0 : \zeta \xtil_t \in \Yc \},
    \end{equation}
    and $v_t = \zeta_t \xtil_t$.
    Then, it follows that:
    \begin{enumerate}
        \item $\zeta_t \in [\gamma_t, 1]$ \label{it:zeta}
        \item $\theta^\top (\xtil_t - v_t) \leq \frac{2 S}{b} \beta_t \| x_t \|_{V_t^{-1}}$ \label{it:vt}
        \item If there exists $\alpha > 0$ such that $x_t = \alpha x_*$, then $v_t = x_*$. \label{it:alph}
        \item If there does not exists $\alpha > 0$ such that $x_t = \alpha x_*$, then $\theta^\top (x_* - v_t) \geq \Delta$. \label{it:alph2}
    \end{enumerate}
\end{lemma}
\begin{proof}
    We condition on $\Econf$ throughout the proof without further reference.
    We will first give some useful facts.
    In particular, it holds that,
    \begin{equation}
    \label{eqn:ucb_pos}
        \hat{\theta}_t^\top \xtil_t + \beta_t \| \xtil_t \|_{V_t^{-1}} \geq \theta^\top x_* > 0
    \end{equation}
    where the first inequality is from \eqref{eqn:optim} and the second is Assumption \ref{ass:bounded}.
    It follows from \eqref{eqn:ucb_pos} that $\xtil_t \neq \mathbf{0}$ and therefore the set $\{ \zeta \geq 0 : \zeta \xtil_t \in \Yc \}$ is compact.
    Also, note that $\{ \zeta \geq 0 : \zeta \xtil_t \in \Yc \}$ contains $0$ and is therefore nonempty, so $\zeta_t$ is well-defined.
    Next, we prove each item individually in the following.

    \textbf{\ref{it:zeta}:} First, we show that $\zeta_t \leq 1$.
    If this were not the case, i.e. $\zeta_t > 1$, then there exists $\zeta' \in \{ \zeta \geq 0 : \zeta \xtil_t \in \Yc \}$ such that $\zeta' > 1$.
    Then, from the definition of $\xtil_t$ in line \ref{lne:opt_act} and the fact that $u = \zeta' \xtil_t$ is in $\Yc \subseteq \Yc_t^o$,
    \begin{equation}
    \label{eqn:contra1}
        \hat{\theta}_t^\top \xtil_t + \beta_t \| \xtil_t \|_{V_t^{-1}} = \max_{x \in \Yc_t^o} \left( \hat{\theta}_t^\top x + \beta_t \| x \|_{V_t^{-1}} \right) \geq \hat{\theta}_t^\top u + \beta_t \| u \|_{V_t^{-1}}.
    \end{equation}
    At the same time, it follows from \eqref{eqn:ucb_pos} that $\hat{\theta}_t^\top \xtil_t + \beta_t \| \xtil_t \|_{V_t^{-1}}$ is positive and therefore,
    \begin{equation}
    \label{eqn:contra2}
        \hat{\theta}_t^\top \xtil_t + \beta_t \| \xtil_t \|_{V_t^{-1}} < \zeta' \left( \hat{\theta}_t^\top \xtil_t + \beta_t \| \xtil_t \|_{V_t^{-1}} \right) = \hat{\theta}_t^\top u + \beta_t \| u \|_{V_t^{-1}}.
    \end{equation}
    Since \eqref{eqn:contra1} and \eqref{eqn:contra2} cannot simultaneously be true, it must hold that $\zeta_t \leq 1$.

    Then, we show that $\zeta_t \geq \gamma_t$.
    Since $x_t = \gamma_t \xtil_t \in \Yc$, it holds that $\gamma_t \in \{ \zeta \geq 0 : \zeta \xtil_t \in \Yc \}$ and therefore $\zeta_t \geq \gamma_t$.

    \textbf{\ref{it:vt}:}
    Since, $v_t = \zeta_t \xtil_t$ and $\zeta_t \in [\gamma_t, 1]$, it follows from Lemma \ref{lem:gam2} that
    \begin{equation*}
        \theta^\top (\xtil_t - v_t) = \theta^\top \xtil_t (1 - \zeta_t) \leq S (1 - \zeta_t) \leq S (1 - \gamma_t) \leq \frac{2 S}{b} \beta_t \| x_t \|_{V_t^{-1}}.
    \end{equation*}

    \textbf{\ref{it:alph}:}
    First, we will show that $\zeta_* = \max\{ \zeta \geq 0 : \zeta x_* \in \Yc \} = 1$.
    If this were not the case, then either $\zeta_* < 1$ or $\zeta_* > 1$.
    The case $\zeta_* < 1$ would imply that $x_*$ is not in $\Yc$, while the case that $\zeta_* > 1$ would imply the existence of a point $x = \zeta x_* \in \Yc$ with $\zeta > 1$ such that $\theta^\top x = \zeta (\theta^\top x_*) > \theta^\top x_*$ (where we use $\theta^\top x_* > 0$ from Assumption \ref{ass:bounded}).
    Either case contradicts the definition of $x_*$ and therefore cannot hold.

    Now, we turn to the statement.
    If there exists $\alpha > 0$ such that $x_t = \alpha x_*$, then,
    \begin{equation*}
        \zeta_t = \max\{ \zeta \geq 0 : \zeta \xtil_t \in \Yc \} = \gamma_t \max\{ \zeta' \geq 0 : \zeta' x_t \in \Yc \} = \frac{\gamma_t}{\alpha} \max\{ \tilde{\zeta} \geq 0 : \tilde{\zeta} x_* \in \Yc \} = \frac{\gamma_t}{\alpha},
    \end{equation*}
    where we use the mapping $\zeta' = \frac{1}{\gamma_t} \zeta$ in the first equality and $\tilde{\zeta} = \alpha \zeta'$ in the second equality.
    Therefore, it follows that
    \begin{equation*}
        v_t = \zeta_t \xtil_t = \frac{\zeta_t}{\gamma_t} x_t = \frac{\alpha \zeta_t}{\gamma_t} x_* = x_*.
    \end{equation*}

    \textbf{\ref{it:alph2}:} First, note that if there does not exist $\alpha > 0$ such that $x_t = \alpha x_*$, then there does not exist $\alpha' > 0$ such that $v_t = \alpha' x_*$ as $v_t = \frac{\zeta_t}{\gamma_t} x_t$.
    Then, since $v_t \in \Yc$, it follows from the definition of $\Delta$ that,
    \begin{equation*}
        \Delta = \inf_{x \in \Yc:\ x \neq \alpha x_*\ \forall \alpha > 0 } \theta^\top (x_* - x) \leq \theta^\top (x_* - v_t).
    \end{equation*}
\end{proof}

Then, we restate and prove Theorem \ref{thm:prob_dep}.

\begin{theorem}[Duplicate of Theorem \ref{thm:prob_dep}]
  \label{thm:prob_dep2}
  Let Assumptions \ref{ass:set_bound}, \ref{ass:bounded} and \ref{ass:noise} hold.
  If $\Delta > 0$, then the number of wrong directions chosen by ROFUL (Algorithm \ref{alg:main_alg}) satisfies
  \begin{equation*}
    B_T \leq \frac{32 d \beta_T^2}{\nu^2 \Delta^2}  \log\left(1 + \frac{T }{\lambda d} \right)
  \end{equation*}
  with probability at least $1 - \delta$.
\end{theorem}

\begin{proof}
  In order to bound the number of times that the wrong action is chosen, we study the regret due to the wrong choice of direction,
  \begin{equation*}
    \tilde{R}_T := \sum_{t=1}^{T} \theta^\top (x_* - v_t),
  \end{equation*}
  where $v_t = \zeta_t \xtil_t$ with $\zeta_t$ defined in \eqref{eqn:a_bound}.
  We will denote the instantaneous regret due to the wrong choice of direction as $\tilde{r}_t = \theta^\top (x_* - v_t)$.
  It follows from Lemma \ref{lem:vt} (\#\ref{it:alph} and \#\ref{it:alph2}) that $\Tilde{r}_t = 0$ if there exists $\alpha > 0$ such that $x_t = \alpha x_*$ (i.e. $x_t$ is in the correct direction) and $\Tilde{r}_t \geq \Delta$ otherwise.
  Therefore, 
  \begin{align*}
    \tilde{R}_T = \sum_{t=1}^{T} \tilde{r}_t & = \sum_{t=1}^{T} \tilde{r}_t \indic\{ \nexists \ \alpha > 0 : x_t = \alpha x_* \}\\
    & \geq \Delta \sum_{t=1}^{T}  \indic\{ \nexists \ \alpha > 0 : x_t = \alpha x_* \}\\
    & = \Delta B_T.
  \end{align*}
  Since $\tilde{R}_T \geq \Delta B_T$ and $\Delta > 0$, an upper bound on $\tilde{R}_T$ implies an upper bound on $B_T$.

  Then, we bound $\tilde{r}_t$ in the following,
  \begin{align*}
    \tilde{r}_t & = \theta^\top x_* - \theta^\top v_t\\
    & \leq \hat{\theta}_t^\top \xtil_t + \beta_t \| \xtil_t \|_{V_t^{-1}} - \theta^\top v_t \label{eqn:dir_reg_a} \tag{a} \\
    & = \theta^\top \xtil_t + (\hat{\theta}_t - \theta)^\top \xtil_t + \beta_t \| \xtil_t \|_{V_t^{-1}} - \theta^\top v_t \\
    & \leq \theta^\top \xtil_t + 2\beta_t \| \xtil_t \|_{V_t^{-1}} - \theta^\top v_t \label{eqn:dir_reg_b} \tag{b} \\
    & = \theta^\top (\xtil_t - v_t) + 2 \beta_t \| \xtil_t \|_{V_t^{-1}}\\
    & \leq \frac{2 S}{b} \beta_t \| x_t \|_{V_t^{-1}} + 2 \beta_t \| \xtil_t \|_{V_t^{-1}} \label{eqn:dir_reg_c} \tag{c} \\
    & \leq \frac{2 S}{b} \beta_t \| x_t \|_{V_t^{-1}} + \frac{2}{\nu} \beta_t \| x_t \|_{V_t^{-1}} \label{eqn:dir_reg_d} \tag{d} \\
    & = \frac{4 S}{b} \beta_t \| x_t \|_{V_t^{-1}},
  \end{align*}
  where \eqref{eqn:dir_reg_a} follows from the fact that $\xtil_t$ is optimistic (i.e. \eqref{eqn:optim}), \eqref{eqn:dir_reg_b} is from the definition of the confidence set, \eqref{eqn:dir_reg_c} is from Lemma \ref{lem:vt} (\#\ref{it:vt}), and \eqref{eqn:dir_reg_d} is due to \eqref{eqn:xt_scale}.

  Since either $\tilde{r}_t \geq \Delta$ or $\tilde{r}_t = 0$, it holds that $\tilde{r}_t \leq (\tilde{r}_t)^2/\Delta$.
  Then, we have that
  \begin{align*}
    \tilde{R}_T & = \sum_{t=1}^T \rtil_t\\
    & \leq \sum_{t=1}^{T} \frac{(\tilde{r}_t)^2}{\Delta}\\
    & \leq \frac{16 S^2 \beta_T^2}{b^2 \Delta} \sum_{t=1}^{T} \| x_t \|_{V_t^{-1}}^2\\
    & \leq \frac{32 S^2 \beta_T^2 d}{b^2 \Delta}  \log\left(1 + \frac{T }{\lambda d} \right)
  \end{align*}
  where the last inequality uses Lemma \ref{lem:sp_elip}.
  Finally, we have that
  \begin{equation*}
    B_T \leq \frac{\tilde{R}_T}{\Delta} \leq \frac{32 S^2 \beta_T^2 d}{b^2 \Delta^2 }  \log\left(1 + \frac{T}{\lambda d} \right).
  \end{equation*}
\end{proof}

Now, we turn out attention to Corollary \ref{thm:prob_dep_reg}.
To do so, we state PD-ROFUL more formally in Algorithm \ref{alg:pd_oful}.
Note that in the second phase of the algorithm, we reduce the problem to a 1-dimensional safe linear bandit problem which is defined formally in the following.

\begin{definition}[Reduction to 1-dimensional problem]
  \label{def:red_1d}
  Given a direction $u_* \in \Sb$, the safe linear bandit problem (Section \ref{sec:pset}) reduces to a 1-dimensional setting.
  For each round $t$ of this setting, the learner chooses $\xi_t \in \Rb_+$ and then $\xi_t u_*$ is played in the original setting (Section \ref{sec:pset}).
\end{definition}

Using this, we give psuedo-code for PD-ROFUL.

\IncMargin{1em}
\begin{algorithm2e}[h]
\caption{Problem-Dependent ROFUL (PD-ROFUL)}
\label{alg:pd_oful}
\DontPrintSemicolon
Set $\Ac = \{\}$, $\bar{N} = 0$ and $t = 1$.\;
\While{$\bar{N} \leq \bar{B} = \frac{32 S^2 d }{b^2 \Delta^2} \beta_T^2 \log\left(1 + \frac{T }{\lambda d} \right)$}{
  ROFUL plays $x_t$ and observes $y_t, z_t$.\;
  \uIf{$x_t \neq \mathbf{0}$ \emph{\textbf{and}} $u_t = x_t / \| x_t \| \not\in \Ac$}{
    $\Ac = \Ac \cup \{ u_t \}$.\;
    $N_{u_t} = 1$.
  }
  \uElseIf{$x_t \neq \mathbf{0}$ \emph{\textbf{and}} $u_t = x_t / \| x_t \| \in \Ac$}{
    $N_{u_t} = N_{u_t} + 1$
  }
  Set $\bar{N} = \max_{u \in \Ac} N_u$.\;
  Set $t = t + 1$.\;
}
Set $u_* = \argmax_{u \in \Ac} N_u$ and $\tau = t$.\;
For $t \in [\tau + 1,T]$, ROFUL is restarted and plays 1-dimensional setting (Definition \ref{def:red_1d}) in direction $u_*$ for remaining rounds.\;
\end{algorithm2e}
\DecMargin{1em}

\begin{corollary}[Duplicate of Corollary \ref{thm:prob_dep_reg}]
  \label{thm:prob_dep_reg2}
  Let Assumptions \ref{ass:set_bound}, \ref{ass:noise} and \ref{ass:bounded} hold.
  With $\Delta > 0$, the regret of PD-ROFUL (Algorithm \ref{alg:pd_oful}) satisfies
  \begin{equation*}
    R_T \leq \frac{4S}{b} \beta_{2 \bar{B} + 1} \sqrt{ 2 d (2 \bar{B} + 1)\log\left(1 + \frac{2 \bar{B} + 1}{\lambda d} \right)} + \frac{4 S}{b} \tilde{\beta}_{T} \sqrt{ 2 T\log\left(1 + \frac{T}{\lambda d} \right)}
  \end{equation*}
  with probability at least $1 - 2 \delta$, where $\tilde{\beta}_{T}$ is $\beta_T$ with $d = 1$.
\end{corollary}
\begin{proof}
  We condition on the confidence sets holding jointly for both the first and second phases, which occurs with probability at least $1- 2 \delta$.

  First, we argue that the optimal direction is correctly identified, i.e. $u_* = x_* / \| x_* \|$.
  Intuitively, this holds because Theorem \ref{thm:prob_dep} says that wrong directions are chosen at most $\bar{B}$ times, so any single direction that is chosen more than this must be the optimal direction.
  Concretely, Theorem \ref{thm:prob_dep} implies that for every single wrong direction $u \in \Sb, u \neq x_* / \| x_* \|$, the actions chosen by ROFUL in the first $\tau - 1$ rounds will satisfy
  \begin{equation*}
      \sum_{t=1}^{\tau - 1} \indic\{ u_t = u \} \leq \sum_{t=1}^{\tau - 1} \indic\{ \nexists \ \alpha > 0 : x_t = \alpha x_* \} \leq B_T \leq \bar{B},
  \end{equation*}
  where we use the notation $u_t = x_t/\| x_t \|$ for nonzero $x_t$ as specified in Algorithm \ref{alg:pd_oful}.
  Then, since $u_*$ is specified in PD-ROFUL to be a direction that ROFUL plays more than $\bar{B}$ times in $\tau - 1$ rounds,
  \begin{equation*}
      \sum_{t=1}^{\tau - 1} \indic\{ u_t = u_* \} > \bar{B},
  \end{equation*}
  and therefore $u_*$ cannot be a wrong direction or equivalently, $u_* = x_* / \| x_* \|$.
  
  Then, we show the bound on the duration of the first phase $\tau - 1 \leq 2 \bar{B} + 1$.
  The intuition is that $u_*$ is played at most $\bar{B} + 1$ times and wrong directions are played at most $\bar{B}$ times, so the total duration of the first phase must be less than $2 \bar{B} + 1$.
  Concretely, it follows from the fact that $u_* = x_* / \| x_* \|$,
    \begin{align*}
        \tau - 1 & = \sum_{t=1}^{\tau - 1} \indic\{ \exists \ \alpha > 0 : x_t = \alpha x_* \} + \sum_{t=1}^{\tau - 1} \indic\{ \nexists \ \alpha > 0 : x_t = \alpha x_* \}\\
        & = \underbrace{\sum_{t=1}^{\tau - 1} \indic\{ u_t = u_* \}}_{\leq \bar{B} + 1} + \underbrace{\sum_{t=1}^{\tau - 1} \indic\{ \nexists \ \alpha > 0 : x_t = \alpha x_* \}}_{\leq \bar{B}} \leq 2 \bar{B} + 1,
    \end{align*}
  where the bound on the first term holds because the first round ends when the correct direction is played more than $\bar{B}$ times and the bound on the second term follows from Theorem \ref{thm:prob_dep}.
  
  Finally, we study the regret.
  To do so, we decompose the regret due to the first and second phases respectively,
  \begin{equation*}
    R_T = \underbrace{\sum_{t=1}^{\tau - 1} r_t}_{R_T^\mathrm{I}} + \underbrace{\sum_{t=\tau}^{T} r_t}_{R_T^\mathrm{II}}.
  \end{equation*}
  Then, we use Theorem \ref{thm:main} and the fact that $\tau - 1 \leq 2 \bar{B} + 1$ to bound $R_T^\mathrm{I}$,
  \begin{equation*}
    R_T^\mathrm{I} = \sum_{t=1}^{\tau-1} r_t \leq \frac{4S}{b} \beta_{\tau - 1} \sqrt{ 2 d (\tau - 1)\log\left(1 + \frac{\tau - 1}{\lambda d} \right)} \leq \frac{4S}{b} \beta_{2 \bar{B} + 1} \sqrt{ 2 d (2 \bar{B} + 1)\log\left(1 + \frac{2 \bar{B} + 1}{\lambda d} \right)},
  \end{equation*}
  For the remainder of the rounds, play of the algorithm is ROFUL with $d=1$.
  Also, the duration of the second phase is less than the time horizon, i.e. $T - (\tau - 1) \leq T$.
  Therefore, $R_T^\mathrm{II}$ is less than that given in Theorem \ref{thm:prob_dep_reg} with $d = 1$.
\end{proof}

\section{Details of Safe-PE Algorithm}
\label{apx:safepe}

\IncMargin{1em}
\begin{algorithm2e}[t]
\caption{Safe Phased Elimination (Safe-PE)}
\label{alg:pe1}
\DontPrintSemicolon
\KwIn{$\Xc, S, b, d, \rho, \delta \in (0,1), \lambda \geq 1$}
Set $\Ac_1 = [k]$, $\zeta_{i,1} = \min(b/S, \alpha_i) \ \forall i \in [k]$, $\Yc_1^p = \{\zeta_{i,1} u_i \}_{i \in [k]}$, $J = \lceil \log_2(T + 1) \rceil$, $t_j = 2^{j-1}$.\;
\For{$j=1$ \KwTo $J$}{
  $V_{t_j} = \lambda I$.\;
  $\tau_j = \min(t_{j+1} - 1, T)$.\;
  \For{$t=t_j$ \KwTo $\tau_j$}{
    Play $x_t \in \argmax_{x \in \Yc_j^p} \| x \|_{V_t^{-1}}$.\label{lne:shrink_width}\;
    $V_{t+1} = V_t + x_t x_t^\top$.\;
  }
  Set $\hat{\theta}_j = \bar{V}_j^{-1} \sum_{s=t_j}^{\tau_j}x_s y_s $ and $\hat{a}_j = \bar{V}_j^{-1} \sum_{s=t_j}^{\tau_j} x_s z_s $, where $\bar{V}_j = V_{\tau_j + 1}$.\;
  Find $\hat{x}_j \in \argmax_{x\in \Yc_j^p} \left( \hat{\theta}_j^\top x - \beta \| x \|_{\bar{V}_j^{-1}} \right)$.\label{lne:max_point}\;
  Update $\Ac_{j+1} = \left\{ i \in \Ac_{j} : \hat{\theta}_j^\top (\hat{x}_j - \zeta_{i,j} u_i) \leq \beta \| \hat{x}_j \|_{\bar{V}_j^{-1}} + \beta \zeta_{i,j} \|u_i\|_{\bar{V}_j^{-1}} +  \frac{2 S  \beta \zeta_{i,j} \|u_i\|_{\bar{V}_{j-1}^{-1}}}{b } \right\} $.\label{lne:elim_dirs}\;
  Update $\mu_{i,j+1} := \max \left\{ \alpha \in [0, \alpha_i] : \alpha \left( u_i^\top \hat{a}_j + \beta \| u_i \|_{\bar{V}_j^{-1}} \right) \leq b \right\}$ and $\zeta_{i, j+1} = \max\left( \zeta_{i, j}, \mu_{i,j+1} \right)$ for all $i \in \Ac_{j+1}$.\label{lne:scaling}\;
  Update $\Yc_{j + 1}^p = \{ \zeta_{i,j+1} u_i \}_{i \in \Ac_{j+1}}$.\label{lne:pess_set}\;
}
\end{algorithm2e}
\DecMargin{1em}

In this section, we give the details of the Safe-PE algorithm (Algorithm \ref{alg:pe1}) discussed in Section \ref{sec:spe}.
This algorithm relies on the action set being a finite star convex set, which we formally assume in the following.

\begin{assumption}
  \label{ass:fstarconv}
  The action set satisfies
  \begin{equation*}
    \Xc := \bigcup_{i \in [k]} \{\alpha u_i : \alpha \in [0,\alpha_i] \},
  \end{equation*}
  where $u_1, ..., u_k \in \Sb$ are unit vectors and $\alpha_1, ..., \alpha_k \in \Rb_{++}$ are the maximum scalings for each unit vectors.
\end{assumption}

The Safe-PE algorithm builds on SpectralEliminator from \citet{valko2014spectral} and \citet{kocak2020spectral}.
It differs in that it eliminates directions in each phase rather than distinct actions and only plays actions from a verifiably safe set.

\subsection{Operation of algorithm}

The Safe-PE algorithm consists of phases $j = 1,2,...$, which are each of duration $2^{j - 1}$.
Throughout its operation, the algorithm maintains a set of direction indexes $\Ac_j$ and safe actions $\Yc_j^p$.
The key parts of each phase are:
\begin{enumerate}
  \item For $2^{j - 1}$ rounds, chooses the action in $\Yc_j^p$ with the largest confidence set width $\| \cdot \|_{V_t^{-1}}$ (line \ref{lne:shrink_width}).
  \item Eliminates directions from $\Ac_j$ with too low of estimated reward (line \ref{lne:elim_dirs}).
  \item Updates $\Yc_j^p$ with the maximum scaling of each direction that is verifiably safe (lines \ref{lne:scaling} and \ref{lne:pess_set}).
\end{enumerate}
The algorithm relies on a confidence set to determine which directions should be eliminated and to ensure that the constraints are not violated.
Different from the confidence set in Lemma \ref{lem:conf_sets}, the radius of the following confidence set does not grow with $d$.
We prove such a confidence set in the following lemma.

\begin{lemma}
  \label{lem:pe_conf}
  Let Assumptions \ref{ass:pe_noise} and \ref{ass:fstarconv} hold and fix some $\delta \in (0,1)$.
  Then, for all $x \in \Xc$ and all $j \in [J]$ it holds that $|x^\top (\hat{\theta}_j - \theta)| \leq \beta \| x \|_{\bar{V}_j^{-1}}$ and $|x^\top (\hat{a}_j - a)| \leq \beta \| x \|_{\bar{V}_j^{-1}}$ where $\beta = \rho \sqrt{2 \log \left( \frac{4 k J}{\delta} \right)} + \sqrt{\lambda} S$ with probability at least $1 - \delta$.
\end{lemma}
\begin{proof}
  First, we find a confidence set that applies for a fixed $u \in \{ u_1, ..., u_k \}$.
  To do so, we start as
  \begin{align*}
    u^\top (\hat{\theta}_j - \theta) & = u^\top \left( \bar{V}_j^{-1} \sum_{t=t_j}^{\tau_j} x_t y_t - \theta \right)\\
    & = u^\top \left( \bar{V}_j^{-1} \sum_{t=t_j}^{\tau_j} x_t (x_t^\top \theta + \epsilon_t) - \theta \right)\\
    & = \underbrace{u^\top \left( \bar{V}_j^{-1} \sum_{t=t_j}^{\tau_j} x_t x_t^\top - I \right) \theta}_{\tone} +  \underbrace{u^\top \bar{V}_j^{-1} \sum_{t=t_j}^{\tau_j} x_t \epsilon_t}_{\ttwo}
  \end{align*}
  Using the notation $\tilde{V} = \sum_{t=t_j}^{\tau_j} x_s x_s^\top$, we study the first term,
  \begin{align*}
   | \tone| & = \left| u^\top \left( \bar{V}_j^{-1} \tilde{V} - I \right) \theta \right|\\
   & = \left| u^\top \bar{V}_j^{-1} \left( \tilde{V} - \bar{V}_j \right) \theta \right|\\
    & = \lambda \left| u^\top \bar{V}_j^{-1} \theta \right| \label{eqn:pe_conf_a} \tag{a} \\
    & \leq \lambda \| \theta \| \| u^\top \bar{V}_j^{-1} \| \\
    & = \lambda \| \theta \| \sqrt{u^\top  \bar{V}_j^{-1} \bar{V}_j^{-1} u} \\
    & \leq \sqrt{\lambda} S \| u \|_{\bar{V}_j^{-1}} \label{eqn:pe_conf_b} \tag{b} ,
  \end{align*}
  where \eqref{eqn:pe_conf_a} is due to the fact that $\tilde{V} - \bar{V}_j = - \lambda I$, and \eqref{eqn:pe_conf_b} is due to the fact that $ \bar{V}_j \succeq \lambda I$ and therefore $y^\top \bar{V}_j^{-1} y \leq \| y \|^2 /  \lambda$ for any $y \in \Rb^d$.

  Now, we look at the second term, which can be written as
  \begin{equation*}
    \ttwo = \sum_{t=t_j}^{\tau_j} \left( u^\top \bar{V}_j^{-1} x_t \epsilon_t \right).
  \end{equation*}
  Since the noise terms $\epsilon_t$ are independent and the actions $x_t$ within each phase (i.e. $t \in [t_j,\tau_j]$) are deterministic given the history at the beginning of the phase, we can use standard concentration of subgaussian random variables (e.g. as in equation 20.2 of \cite{lattimore2020bandit}) to get that 
  \begin{equation*}
    |\ttwo| \leq \rho \sqrt{2\log\left(\frac{2}{\delta'} \right) \sum_{t=t_j}^{\tau_j} (u^\top \bar{V}_j^{-1} x_s)^2}
  \end{equation*}
  with probability at least $1 - \delta'$ for some $\delta' \in (0,1)$.
  Also, since $\tilde{V} \preceq \bar{V}_j$,
  \begin{equation*}
    \sum_{t=t_j}^{\tau_j} (u^\top \bar{V}_j^{-1} x_t)^2 = u^\top \bar{V}_j^{-1} \left( \sum_{t=t_j}^{\tau_j} x_t x_t^\top \right)\bar{V}_j^{-1} u = u^\top \bar{V}_j^{-1} \tilde{V} \bar{V}_j^{-1} u \leq u^\top \bar{V}_j^{-1} \bar{V}_j \bar{V}_j^{-1} u = \| u \|_{\bar{V}_j^{-1}}^2.
  \end{equation*}
  Therefore, with probability at least $1 - \delta'$, it holds that
  \begin{equation*}
    |u^\top (\hat{\theta}_j - \theta)| \leq | \tone | + | \ttwo | \leq \left( \rho \sqrt{2 \log\left(\frac{2}{\delta'} \right)} + \sqrt{\lambda} S \right) \| u \|_{\bar{V}_j^{-1}}.
  \end{equation*}
  Note that the same applies replacing $\theta$ with $a$.
  Then, by taking
  \begin{equation*}
    \beta = \rho \sqrt{2 \log \left( \frac{4 k J}{\delta} \right)} + \sqrt{\lambda} S,
  \end{equation*}
  it holds that $|u^\top (\hat{\theta}_j - \theta)| \leq \beta \| u \|_{\bar{V}_j^{-1}}$ and $|u^\top (\hat{a}_j - a)| \leq \beta \| u \|_{\bar{V}_j^{-1}}$ for all $u \in \{ u_1, ..., u_k \}$ and all $j \in [J]$ with probability at least $1 - \delta$.
  Since all $x \in \Xc$ can be written as $\alpha u$ for some $\alpha \geq 0$ and $u \in \{ u_1, ..., u_k \}$, it holds under the same conditions that
  \begin{equation*}
    |x^\top (\hat{\theta}_j - \theta)| = \alpha |u^\top (\hat{\theta}_j - \theta)| \leq \alpha \beta \| u \|_{\bar{V}_j^{-1}} = \beta \| x \|_{\bar{V}_j^{-1}}.
 \end{equation*}
\end{proof}

We define the event that this confidence set holds as
\begin{equation*}
  \Econf := \left\{ |x^\top (\hat{\theta}_j - \theta)| \leq \| x \|_{\bar{V}_j^{-1}} \beta, |x^\top (\hat{a}_j - a)| \leq \| x \|_{\bar{V}_j^{-1}} \beta \quad \forall x \in \Xc\ \forall j \in [J] \right\},
\end{equation*}
which occurs with probability at least $1 - \delta$.

\subsection{Regret analysis}

Now, we will prove the regret bound for Safe-PE.
In order to do so, we need some more notation.
The true maximum scaling for each direction is denoted by
\begin{equation}
  \bar{\zeta}_{i} := \max \left\{ \alpha \in [0, \alpha_i] : \alpha u_i^\top a \leq b \right\}.
\end{equation}
Also, let $v_{i,j} := \zeta_{i, j} u_{i}$ and $\bar{v}_i := \bar{\zeta}_i u_i$.
The index of the direction played at round $t$ and the optimal direction are denoted by $i_t$ and $i_*$, respectively.
When $i_t$ or $i_*$ are used in a subscript, the shorthand $t$ and $*$ are used.
With this, we prove a bound on the safe scalings.

\begin{lemma}
  \label{lem:gam_pe}
  Let Assumptions \ref{ass:set_bound}, \ref{ass:bounded} and \ref{ass:fstarconv}  hold and assume that $\Econf$ holds.
  For all $i \in [k]$ and $j \geq 1$, it holds that
  \begin{equation*}
    \frac{\zeta_{i,j}}{\bar{\zeta}_i} \geq 1 - \frac{2}{b} \| v_{i,j} \|_{\bar{V}_{j-1}^{-1}} \beta
  \end{equation*}
  and furthermore that
  \begin{equation*}
    \frac{\zeta_{i,j}}{\bar{\zeta}_i} \geq 1 - \frac{2}{b^2} S \| v_{i,j-1} \|_{\bar{V}_{j-1}^{-1}} \beta.
  \end{equation*}
\end{lemma}
\begin{proof}
  We condition on $\Econf$ throughout the proof without further reference.
  We aim to find a $\gamma \geq 0$ such that $v_{i,j}  = \gamma \bar{v}_i $.
  First, we show that $\mu \bar{v}_i$ is in $\Yc_j^p$, where $\mu = \frac{b}{2 \| \bar{v}_i \|_{\bar{V}_{j-1}^{-1}} \beta + b }$.
  This holds because,
  \begin{align*}
    \hat{a}_{j-1}^\top (\mu \bar{v}_i) + \beta \| \mu \bar{v}_i \|_{\bar{V}_{j-1}^{-1}} & = \mu \left( \hat{a}_{j-1}^\top \bar{v}_i + \beta \| \bar{v}_i \|_{\bar{V}_{j-1}^{-1}} \right)\\
    & = \mu \left( a^\top \bar{v}_i + (\hat{a}_{j-1}^\top - a)^\top \bar{v}_i + \beta \| \bar{v}_i \|_{\bar{V}_{j-1}^{-1}} \right)\\
    & \leq \mu \left( a^\top \bar{v}_i + 2 \beta \| \bar{v}_i \|_{\bar{V}_{j-1}^{-1}} \right)\\
    & \leq \mu \left( b + 2 \beta \| \bar{v}_i \|_{\bar{V}_{j-1}^{-1}} \right) = b
  \end{align*}
  where the first inequality is from the definition of $\Econf$ and the second inequality is due to the fact that $\bar{v}_i$ satisfies the constraints by definition.
  It follows that $\gamma \geq \frac{b}{2 \| \bar{v}_i \|_{\bar{V}_{j-1}^{-1}} \beta + b }$.
  Then, using the fact that $\gamma \| \bar{v}_i \|_{\bar{V}_{j-1}^{-1}} = \| \gamma \bar{v}_i \|_{\bar{V}_{j-1}^{-1}} = \| v_{i,j} \|_{\bar{V}_{j-1}^{-1}}$, we can rearrange this to get that
  \begin{equation*}
    \frac{\zeta_{i,j}}{\bar{\zeta}_i} = \gamma \geq 1 - \frac{2 \| v_{i,j} \|_{\bar{V}_{j-1}^{-1}} \beta}{b}
  \end{equation*}
  This is the first claim.

  Also, we know from the definition of $\zeta_{i,j+1}$ that
  \begin{equation*}
    \alpha_i \geq \zeta_{i,j+1} \geq \zeta_{i,j} \geq \zeta_{i,0} = \min \left( \frac{b}{S}, \alpha_i \right),
  \end{equation*}
  and therefore
  \begin{equation}
    \label{eqn:triv_scale}
    \frac{\zeta_{i,j}}{\zeta_{i,j-1}} \leq \frac{\alpha_i}{\zeta_{i,0}} = \frac{\alpha_i}{\min \left( \frac{b}{S}, \alpha_i \right)} = \frac{1}{\min \left( \frac{b}{S \alpha_i}, 1 \right)} \leq \frac{1}{\min \left( \frac{b}{S}, 1 \right)} = \frac{S}{b},
  \end{equation}
  where we use $\frac{b}{S} \leq 1$ from Assumption \ref{ass:bounded}.
  It follows that
  \begin{equation*}
    \frac{\zeta_{i,j}}{\bar{\zeta}_i} \geq 1 - \frac{2 \| v_{i,j} \|_{\bar{V}_{j-1}^{-1}} \beta}{b} = 1 - \frac{2 (\zeta_{i,j}/\zeta_{i,j-1}) \| v_{i,j-1} \|_{\bar{V}_{j-1}^{-1}} \beta}{b} \geq  1 - \frac{2 S \| v_{i,j-1} \|_{\bar{V}_{j-1}^{-1}} \beta}{b^2},
  \end{equation*}
  giving the second claim.
\end{proof}

Next, we show that the optimal action is never eliminated with high probability.

\begin{lemma}
  \label{lem:opt_dir}
  Let Assumptions \ref{ass:set_bound}, \ref{ass:bounded} and \ref{ass:fstarconv}  hold and suppose that $\Econf$ holds.
  It follows that $i_*$ is in $\Ac_j$ for all $j \in [J]$.
\end{lemma}
\begin{proof}
  First note that, given Lemma \ref{lem:gam_pe}, it holds that
  \begin{equation}
    \label{eqn:opt_scaling}
    \begin{split}
      \theta^\top (x_* - v_{*,j}) & = \theta^\top u_* ( \bar{\zeta}_* - \zeta_{*,j})\\
      & = \bar{\zeta}_* \theta^\top u_* (1 - \zeta_{*,j}/\bar{\zeta}_* )\\
      & \leq \bar{\zeta}_* S (1 - \zeta_{*,j}/\bar{\zeta}_* )\\
      & \leq \frac{2 S \| v_{*,j} \|_{\bar{V}_{j-1}^{-1}} \beta}{b}
    \end{split}
  \end{equation}
  Then, conditioning on $\Econf$, we have for all $j \in [J]$ that
  \begin{equation}
    \label{eqn:opt_ret}
    \begin{split}
    \hat{\theta}_j^\top( \hat{x}_j - v_{*,j}) &= \hat{x}_j^\top ( \hat{\theta}_j - \theta) + \theta^\top \hat{x}_j - \theta^\top v_{*,j} + v_{*,j}^\top ( \theta - \hat{\theta}_j)\\
    & \leq \hat{x}_j^\top ( \hat{\theta}_j - \theta) + \theta^\top (x_* - v_{*,j}) + v_{*,j}^\top ( \theta - \hat{\theta}_j)\\
    & \leq \beta \| \hat{x}_j \|_{\bar{V}_j^{-1}} + \frac{2 S \| v_{*,j} \|_{\bar{V}_{j-1}^{-1}} \beta}{b} + \beta \| v_{*,j} \|_{\bar{V}_j^{-1}},
    \end{split}
  \end{equation}
  where the first inequality comes from the optimality of $x_*$ (i.e. $\theta^\top \hat{x}_j \leq \theta^\top x_*$) and the second inequality applies the confidence set to the first and third terms and \eqref{eqn:opt_scaling} to the second term.
  Note that \eqref{eqn:opt_ret} is exactly the condition for directions to be retained by the algorithm in line \ref{lne:elim_dirs}.
  Therefore, if $i_*$ is in $\Ac_j$ for some $j$, then it will be in $\Ac_{j+1}$.
  Then, since $i_* \in \Ac_1$, it holds that $i_* \in \Ac_j$ for all $j \in [J]$ by induction.
\end{proof}

Next, we relate the actions in $\Yc_j^p$ to the chosen actions.

\begin{lemma}
  \label{lem:max_width}
  For all $j \in [J - 1]$, it holds that
  \begin{equation}
    \label{eqn:wjs}
    w_j := \max_{x \in \Yc_j^p} \| x \|_{\bar{V}_{j}^{-1}} \leq \frac{1}{t_{j+1} - t_j} \sum_{t=t_j}^{t_{j+1} - 1} \| x_t \|_{V_t^{-1}}.
  \end{equation}
\end{lemma}
\begin{proof}
  This proof essentially follows from Lemma 39 in \cite{kocak2020spectral} but we give it for completeness.
  Note that for any $t \in [t_j, \tau_j]$, it holds that $\bar{V}_j = \sum_{s=t + 1}^{\tau_j} x_s x_s^\top + V_t \succeq V_t \succ 0$ and therefore that $\| x \|_{\bar{V}_j^{-1}} \leq \| x \|_{V_{t}^{-1}}$ for any $x$.
  Also, since $j \leq J - 1$, it holds that $\tau_j = t_{j+1} - 1$.
  It follows for any $x \in \Yc_j^p$ that
  \begin{align*}
    (t_{j+1} - t_j) \| x \|_{\bar{V}_j^{-1}} & \leq \sum_{t=t_j}^{t_{j+1} - 1} \| x \|_{V_t^{-1}} \\
    & \leq \sum_{t=t_j}^{t_{j+1} - 1} \max_{x \in \Yc_j^p} \| x \|_{V_t^{-1}}\\
    & = \sum_{t=t_j}^{t_{j+1} - 1} \| x_t \|_{V_t^{-1}}.
  \end{align*}
\end{proof}

Lastly, we put everything together and prove the complete regret bound for the Safe-PE in Theorem~\ref{thm:safepe2}, which shows that the regret is $\Oc(\frac{1}{b^2}\sqrt{d T \log(T) \log(k\log(T))} )$.

\begin{theorem}[Complete version of Theorem \ref{thm:safepe}]
  \label{thm:safepe2}
  Let Assumptions \ref{ass:set_bound}, \ref{ass:bounded}, \ref{ass:pe_noise} and \ref{ass:fstarconv} hold.
  Then, the regret of Safe-PE (Algorithm \ref{alg:pe1}) satisfies
  \begin{equation*}
    R_T  \leq  6 S + 5 \beta \left( \frac{24 S^2   }{b^2} + 10 \right) \sqrt{2 d T \log \left(1 + \frac{T}{\lambda d} \right)}
  \end{equation*}
  with probability at least $1 - \delta$.
\end{theorem}
\begin{proof}
  Without further reference, we condition on $\Econf$ throughout.
  We decompose the instantaneous regret for $t \in [t_j,\tau_j]$ as
  \begin{equation*}
    r_t = \theta^\top x_* - \theta^\top x_t = \underbrace{\theta^\top x_* - \theta^\top v_{*,j-1}}_{\tone} + \underbrace{\theta^\top v_{*,j-1} - \theta^\top v_{t,j-1}}_{\ttwo} + \underbrace{\theta^\top v_{t,j-1} - \theta^\top x_t.}_{\tthree}
  \end{equation*}
  We study each of the terms individually in the following.

  \textbf{Term I:}
  It follows from Lemma \ref{lem:gam_pe} that
  \begin{equation*}
    \tone = \theta^\top (x_* - v_{*,j-1}) = \bar{\zeta}_* \theta^\top u_* (1 - \zeta_{*,j-1}/\bar{\zeta}_*) \leq \frac{2 S^2 \beta}{b^2}  \| v_{*,j-2} \|_{\bar{V}_{j-2}^{-1}}.
  \end{equation*}

  \textbf{Term II:}
  We further decompose $\ttwo$ as 
  \begin{align*}
    \ttwo & = \theta^\top v_{*,j-1} - \theta^\top v_{t,j-1}\\
    & = \underbrace{\hat{\theta}_{j-1}^\top v_{*,j-1} - \hat{\theta}_{j-1}^\top v_{t,j-1} - \beta \| v_{*,j-1} \|_{\bar{V}_{j-1}^{-1}}}_{\ttwoa} + \underbrace{(\theta - \hat{\theta}_{j-1})^\top v_{*,j-1} + (\hat{\theta}_{j-1} - \theta)^\top v_{t,j-1} + \beta \| v_{*,j-1} \|_{\bar{V}_{j-1}^{-1}}}_{\ttwob}
  \end{align*}
  Then, we have that
  \begin{align*}
    \ttwoa & = \hat{\theta}_{j-1}^\top v_{*,j-1} - \hat{\theta}_{j-1}^\top v_{t,j-1} - \beta \| v_{*,j-1} \|_{\bar{V}_{j-1}^{-1}}\\
    & \leq \hat{\theta}_{j-1}^\top \hat{x}_{j-1} - \hat{\theta}_{j-1}^\top v_{t,j-1} - \beta \| \hat{x}_{j-1} \|_{\bar{V}_{j-1}^{-1}}\\
    & \leq 2 \beta \| v_{t,j - 1} \|_{\bar{V}_{j - 1}^{-1}} +  \frac{2 S  \beta}{b} \| v_{t,j - 1} \|_{\bar{V}_{j - 2}^{-1}} \\
    & \leq 2 \beta \| v_{t,j - 1} \|_{\bar{V}_{j - 1}^{-1}} +  \frac{2 S^2  \beta}{b^2 } \| v_{t,j - 2} \|_{\bar{V}_{j - 2}^{-1}}
  \end{align*}
  where the first inequality follows from the definition of $\hat{x}_j$ in line \ref{lne:max_point} of Algorithm \ref{alg:pe1} given that $v_{*,j-1} \in \Yc_{j-1}^p$, the second inequality comes from the fact that $i_t \in \Ac_j$ and therefore satisfied the criteria in line \ref{lne:elim_dirs}, and the third inequality comes from \eqref{eqn:triv_scale}.

  Also, we have that
  \begin{align*}
    \ttwob & = (\theta - \hat{\theta}_{j-1})^\top v_{*,j-1} + (\theta - \hat{\theta}_{j-1})^\top v_{t,j-1} + \beta \| v_{*,j-1} \|_{\bar{V}_{j-1}^{-1}}\\
    & \leq 2 \beta \| v_{*,j-1} \|_{\bar{V}_{j-1}^{-1}} + \beta \| v_{t,j-1} \|_{\bar{V}_{j-1}^{-1}}
    % & \leq 3 \beta w_{j-1},
  \end{align*}
  where the inequality is from the definition of $\Econf$.

  Putting everything together, we have that
  \begin{equation*}
    \ttwo \leq 3 \beta \| v_{t,j - 1} \|_{\bar{V}_{j - 1}^{-1}} + 2 \beta \| v_{*,j-1} \|_{\bar{V}_{j-1}^{-1}} + \frac{2 S^2  \beta}{b^2 } \| v_{t,j - 2} \|_{\bar{V}_{j - 2}^{-1}}.
  \end{equation*}

  \textbf{Term III:}
  We have that
  \begin{align*}
    \tthree & = \theta^\top (v_{t,j-1} - x_t)\\
    & = \theta^\top u_t (\zeta_{t,j-1} - \zeta_{t,j})\\
    & \leq S |\zeta_{t,j-1} - \zeta_{t,j} |\\
    & = S (\zeta_{t,j} - \zeta_{t,j-1}) \tag{a} \label{eqn:tth_a}\\
    & \leq S (\bar{\zeta}_t - \zeta_{t,j-1}) \tag{b} \label{eqn:tth_b}\\
    & = S \bar{\zeta}_t (1 - \zeta_{t,j-1} / \bar{\zeta}_t)\\
    & \leq \frac{2 S^2 \beta}{b^2} \|v_{t,j-2}\|_{\bar{V}_{j-2}^{-1}} \tag{c} \label{eqn:tth_c}
  \end{align*}
  where \eqref{eqn:tth_a} is from the fact that $\zeta_{t,j} \geq \zeta_{t,j-1}$ since $\zeta_{i,j}$ is monotone in $j$ by definition (see line \ref{lne:scaling}), \eqref{eqn:tth_b} is due to the fact that $\zeta_{t,j} \leq \bar{\zeta}_t$ give that $\Yc_j^p \subseteq \Yc$ (conditioned on $\Econf$), and \eqref{eqn:tth_c} is from Lemma \ref{lem:gam_pe}.

  \textbf{Completing the proof:}
  Using the bounds established for each of the terms, it holds that
  \begin{align*}
    r_t & = \tone + \ttwo + \tthree\\
    & \leq 3 \beta \| v_{t,j - 1} \|_{\bar{V}_{j - 1}^{-1}} + 2 \beta \| v_{*,j-1} \|_{\bar{V}_{j-1}^{-1}} + \frac{2 S^2 \beta}{b^2}  \| v_{*,j-2} \|_{\bar{V}_{j-2}^{-1}} + \frac{4 S^2 \beta}{b^2} \|v_{t,j-2}\|_{\bar{V}_{j-2}^{-1}}\\
    & \leq 5 \beta w_{j-1} + \frac{6 S^2 \beta}{b^2} w_{j-2},
  \end{align*}
  where we use the fact $i_* \in \Ac_{j-1} \subseteq \Ac_{j-2}$ (due to Lemma \ref{lem:opt_dir}) and $i_t \in \Ac_j \subseteq \Ac_{j-1} \subseteq \Ac_{j-2}$ which implies that $v_{*,j-1}, v_{t,j-1} \in \Yc_{j - 1}^p$ and $v_{*,j-2}, v_{t,j-2} \in \Yc_{j - 2}^p$.
  Therefore $\| v_{t,j - 1} \|_{\bar{V}_{j - 1}^{-1}} \leq w_{j-1}$, $\| v_{*,j - 1} \|_{\bar{V}_{j - 1}^{-1}} \leq w_{j-1}$, $\| v_{t,j - 2} \|_{\bar{V}_{j - 2}^{-1}} \leq w_{j-2}$ and $\| v_{*,j - 2} \|_{\bar{V}_{j - 2}^{-1}} \leq w_{j-2}$ (where $w_j$ is defined in \eqref{eqn:wjs}).
  
  Then, we study the regret within a single phase $j \geq 3$,
  \begin{align*}
    \sum_{t = t_j}^{\tau_j} r_t & \leq (\tau_{j} - t_j + 1) \left( \frac{6 S^2 \beta w_{j-2} }{b^2} + 5 \beta w_{j-1} \right)\\
    & \leq (t_{j+1} - t_j) \left( \frac{6 S^2 \beta w_{j-2} }{b^2} + 5 \beta w_{j-1} \right) \tag{a} \label{eqn:reg_ph_a}\\
    & = \frac{24 S^2 \beta  }{b^2} (t_{j-1} - t_{j-2}) w_{j-2}  + 10 \beta (t_j - t_{j-1}) w_{j-1} \tag{b} \label{eqn:reg_ph_b} \\
    & \leq \frac{24 S^2 \beta  }{b^2} \sum_{t=t_{j-2}}^{t_{j-1} - 1} \| x_t \|_{V_t^{-1}} + 10 \beta \sum_{t=t_{j-1}}^{t_{j} - 1} \| x_t \|_{V_t^{-1}} \tag{c} \label{eqn:reg_ph_c} \\
    & \leq \frac{24 S^2 \beta  }{b^2} \sqrt{t_{j-2} \sum_{t=t_{j-2}}^{t_{j-1} - 1} \| x_t \|_{V_t^{-1}}^2} + 10 \beta \sqrt{t_{j-1} \sum_{t=t_{j-1}}^{t_{j} - 1} \| x_t \|_{V_t^{-1}}^2} \tag{d} \label{eqn:reg_ph_d} \\
    & \leq \frac{24 S^2 \beta  }{b^2} \sqrt{2 d t_{j-2} \log \left(1 + \frac{t_{j-2}}{\lambda d} \right)} + 10 \beta \sqrt{2 d t_{j-1} \log \left( 1 + \frac{t_{j-1}}{\lambda d} \right)} \tag{e} \label{eqn:reg_ph_e},
  \end{align*}
  where \eqref{eqn:reg_ph_a} is due to the fact that $\tau_j \leq t_{j+1} - 1$, \eqref{eqn:reg_ph_b} is due to the fact that $t_{j+1} - t_{j} = 2(t_{j} - t_{j -1})$, \eqref{eqn:reg_ph_c} is due to Lemma \ref{lem:max_width}, \eqref{eqn:reg_ph_d} is Cauchy-Schwarz and \eqref{eqn:reg_ph_e} is from Lemma \ref{lem:sp_elip}.

  Finally, we can apply this to the total regret as
  \begin{align*}
    R_T & = \sum_{j=1}^{J} \sum_{t=t_j}^{\tau_j} \theta^\top (x_* - x_t) \\
    & \leq 6 S + \sum_{j=3}^{J} \sum_{t=t_j}^{\tau_j} \theta^\top (x_* - x_t) \\
    & \leq 6 S + \sum_{j=3}^{J} \left( \frac{24 S^2 \beta  }{b^2} \sqrt{2 d t_{j-2} \log \left(1 + \frac{t_{j-2}}{\lambda d} \right)} + 10 \beta \sqrt{2 d t_{j-1} \log \left( 1 + \frac{t_{j-1}}{\lambda d} \right)} \right)\\
    & \leq 6 S + \beta \left( \frac{24 S^2   }{b^2} + 10 \right) \sqrt{2 d \log \left(1 + \frac{T}{\lambda d} \right)} \sum_{j=3}^{J} \sqrt{t_{j}}\\
    & \leq 6 S + 5 \beta \left( \frac{24 S^2   }{b^2} + 10 \right) \sqrt{2 d T \log \left(1 + \frac{T}{\lambda d} \right)},
  \end{align*}
  where the last inequality uses $\sum_{j=1}^{J} \sqrt{t_j} = \sum_{j=1}^{J} (\sqrt{2})^j \leq 5 \sqrt{T}$.
\end{proof}

\section{Proofs for linked convex constraints}
\label{apx:convex}

In this section, we prove the regret guarantees for the setting with linked convex constraints.
First, we give some notation and specialize the assumptions from the original setting to this setting.
We denote the vector formed from the $i$th row of $A$ as $a_i$ such that $A = [a_1\ ... \ a_n]^\top$, and the $i$th element of $z_t$ as $z_{t,i}$.
\begin{assumption}
\label{ass:bounded2}
    There exists positive reals $S_A$ and $S_\theta$ such that $\| a_i \| \leq S_A$ for all $i \in [n]$ and $\| \theta \| \leq S_\theta$. Let $S := \max(S_A, S_\theta)$.
    Also, there exists positive real $r$ such that $r \Bb \subseteq \Gc$.
    Lastly, it holds that $\nu := \frac{r}{\sqrt{n} S_A} \leq 1$.\footnote{If $\nu > 1$, then for all $x \in \Xc$ it holds that $\| A x \| \leq \| A \|_F \| x \| < \sqrt{n} S \nu = r$ given that $\| x \| \leq 1 < \nu$ for all $x \in \Xc$. }
\end{assumption}
In the following subsections will first study ROFUL in this setting and then GenOP and Safe-PE.

\subsection{ROFUL under linked convex constraints}

We first update the definitions of ROFUL to this setting, then will prove the regret bounds.
We define the estimator of the vector $a_i$ as
\begin{equation*}
  \hat{a}_{t,i} =  V_t^{-1} \sum_{k=1}^{t-1} x_k z_{k,i}
\end{equation*}
and $\hat{A}_t = [\hat{a}_{t,1}\ ... \ \hat{a}_{t,n}]^\top$.
We then state the specific structural assumption on the noise terms.
\begin{assumption}
  \label{ass:noise2}
    For all $t \in [T]$, it holds that $\Eb[\epsilon_t | x_1, \epsilon_1, ..., \epsilon_{t-1}, x_t] = 0$ and $\Eb[\exp(\lambda \epsilon_t) | x_1, \epsilon_1, ..., \epsilon_{t-1}, x_t] \leq \exp(\frac{\lambda^2 \rho^2}{2}), \forall \lambda \in \Rb$.
    The same holds replacing $\epsilon_t$ with $\eta_{t,i}$ for each $i \in [n]$.
\end{assumption}
With this, we give a generalization of the confidence sets originally defined in Lemma \ref{lem:conf_sets}.
\begin{lemma}[Theorem 2 in \citet{abbasi2011improved}]
  \label{lem:conf_sets2}
  Let Assumptions \ref{ass:set_bound}, \ref{ass:bounded2} and \ref{ass:noise2} hold.
  Also, let
  \begin{equation}
    \beta_t := \rho \sqrt{d \log \left( \frac{1 + (t-1) / \lambda}{\delta / (n + 1)} \right) } + \sqrt{\lambda} S.
  \end{equation}
  Then with probability at least $1 - \delta$, it holds that both $| x^\top (\hat{\theta}_t - \theta)| \leq \beta_t \| x \|_{V_t^{-1}}$ and $(\hat{A}_t - A) x \in \beta_t \| x \|_{V_t^{-1}} \Bb_\infty $ for all $x \in \Rb^d$ and all $t \geq 1$.
\end{lemma}
We use $\Econf$ to refer to the event that the confidence sets in Lemma \ref{lem:conf_sets2} hold for all rounds.
The optimistic and pessimistic sets then become
\begin{equation}
  \label{eqn:opt_set_conv}
  \Yc_t^o = \{ x \in \Xc : \hat{A}_t x + \beta_t \| x \|_{V_t^{-1}} \Bb_\infty  \cap \Gc \neq \emptyset \}
\end{equation}
and
\begin{equation}
  \label{eqn:pess_set_conv}
  \Yc_t^p = \{ x \in \Xc : \hat{A}_t x + \beta_t \| x \|_{V_t^{-1}} \Bb_\infty  \subseteq \Gc \}.
\end{equation}
The main challenge in this setting is characterizing the scaling required to take any point in $\Yc_t^o$ in to $\Yc_t^p$, which we lower bound in the following lemma.

\begin{lemma}
  \label{lem:gamma}
  Let Assumption \ref{ass:set_bound} hold.
  Also, let $x$ be any point in $\Yc_t^o$ and $\zeta = \max\{ \mu \in [0,1] : \mu x \in \Yc_t^p \}$. 
  Then, for all $t$, it holds that
  \begin{equation*}
    \zeta \geq  \frac{r}{r + 2 \sqrt{n} \beta_t \| x \|_{V_t^{-1}}},
  \end{equation*}
  and, with $\bar{x} = \zeta x$, that
  \begin{equation*}
    \zeta \geq 1 -  \frac{2 \sqrt{n}}{r} \beta_t \| \bar{x} \|_{V_t^{-1}}.
  \end{equation*}
\end{lemma}

\begin{proof}
  From the definition of $\Yc_t^o$, we can choose a $v \in \Bb_\infty$ such that
  \begin{equation*}
      u := \hat{A}_t x +  \beta_t \| x \|_{V_t^{-1}} v \in \Gc.
  \end{equation*}
  For $\alpha \in [0,1]$, we know that
  \begin{align*}
      \hat{A}_t \alpha x + \beta_t \| \alpha x \|_{V_t^{-1}} \Bb_{\infty} & \subseteq \hat{A}_t \alpha x + \beta_t \| \alpha x \|_{V_t^{-1}} \left(  2 \Bb_{\infty} + v  \right)\\
      & = \hat{A}_t \alpha x + \beta_t \| \alpha x \|_{V_t^{-1}} v + 2 \beta_t \| \alpha x \|_{V_t^{-1}} \Bb_{\infty} \\
      & = \alpha \left( \hat{A}_t x + \beta_t \| x \|_{V_t^{-1}} v \right) + 2 \beta_t \| \alpha x \|_{V_t^{-1}} \Bb_{\infty} \\
      & = \alpha u + 2 \beta_t \| \alpha x \|_{V_t^{-1}} \Bb_{\infty}.
  \end{align*}
  From Fact \ref{prop:shrunk} and the fact that $u$ is in $\Gc$, we know that $\alpha u + (1 - \alpha) r \Bb \subseteq \Gc$.
  We choose $\alpha = \frac{r}{r + 2 \sqrt{n} \beta_t \| x \|_{V_t^{-1}}}$ such that $\alpha = \frac{r}{\sqrt{n} 2 \beta_t \| x \|_{V_t^{-1}}} ( 1 - \alpha )$ to get that
  \begin{align*}
      \hat{A}_t \alpha x + \beta_t \| \alpha x \|_{V_t^{-1}} \Bb_{\infty} & \subseteq \alpha u + 2 \beta_t \| \alpha x \|_{V_t^{-1}} \Bb_{\infty}\\
      & \subseteq \alpha u + 2 \sqrt{n}  \beta_t \| \alpha x \|_{V_t^{-1}} \Bb\\
      & = \alpha u + r (1 - \alpha ) \Bb \subseteq \Gc.
  \end{align*}
  Since $\hat{A}_t \alpha x + \beta_t \| \alpha x \|_{V_t^{-1}} \Bb_{\infty} \subseteq \Gc$ and $\alpha x$ is in $\Xc$ due to the fact that it is star-convex, we know that $\alpha x \in \Yc_t^p$.
  It follows that
  \begin{equation}
    \label{eqn:zeta_scale}
      \zeta = \max \left\{ \mu \geq 0 : \mu x \in \Yc_t^p \right\} \geq \alpha = \frac{r}{r + 2 \sqrt{n} \beta_t \| x \|_{V_t^{-1}}},
  \end{equation}
  which proves the first inequality in the statement of the lemma.
  Then, given that $\zeta x = \bar{x}$ and $\zeta \geq 0$, it holds that
  \begin{equation*}
    \zeta \| x \|_{V_t^{-1}} = \| \zeta x \|_{V_t^{-1}} = \| \bar{x} \|_{V_t^{-1}}.
  \end{equation*}  
  With this, we can rearrage \eqref{eqn:zeta_scale} to get that 
  \begin{equation*}
    \zeta r + 2 \sqrt{n} \beta_t \| \bar{x} \|_{V_t^{-1}} \geq r \quad \Longleftrightarrow \quad \zeta \geq 1 - \frac{2 \sqrt{n}}{r} \beta_t \| \bar{x} \|_{V_t^{-1}},
  \end{equation*}
  which proves the second inequality in the statement of the lemma.
\end{proof}

The regret bound for ROFUL in this setting then follows from this.

\begin{theorem}
  Let Assumptions \ref{ass:set_bound}, \ref{ass:bounded2} and \ref{ass:noise2} hold.
  Then, the regret of ROFUL in the setting with linked convex constraints satisfies
  \begin{equation*}
    R_T \leq \frac{2 \sqrt{n} }{r} \left( S_\theta + S_A \right) \beta_{T} \sqrt{ 2 d T\log\left(1 + \frac{T}{\lambda d} \right)}
  \end{equation*}
  with probability at least $1 - \delta$.
\end{theorem}
\begin{proof}
  We condition on $\Econf$ throughout the proof without further explicit reference to it.
  From Lemma \ref{lem:gamma} and using the same reasoning as Lemma \ref{lem:gam2}, we know that
  \begin{equation*}
    \gamma_t \geq \max \left(1 - \frac{2 \sqrt{n}}{r} \beta_t \| x_t \|_{V_t^{-1}}, \nu \right).
  \end{equation*}
  Then, we know that
  \begin{align*}
    \theta^\top (\xtil_t - x_t) & \leq S_\theta (1 - \gamma_t)\\
    & \leq \frac{2 \sqrt{n} S_\theta }{r} \beta_t \| x_t \|_{V_t^{-1}}
  \end{align*}
  Also, it holds that
  \begin{align*}
    \theta^\top (x_* - \xtil_t) & \leq \hat{\theta}_t^\top \xtil_t + \beta_t \| \xtil_t \|_{V_t^{-1}} - \theta^\top \xtil_t\\
    & \leq 2 \beta_t \| \xtil_t \|_{V_t^{-1}}\\
    & \leq \frac{2}{\nu} \beta_t \| x_t \|_{V_t^{-1}}
  \end{align*}
  Then, the instantaneous regret satisfies
  \begin{align*}
    r_t & = \theta^\top (x_* - x_t)\\
    & \leq \frac{2 \sqrt{n} }{r} \left( S_\theta + S_A \right) \beta_t \| x_t \|_{V_t^{-1}}.
  \end{align*}
  The proof follows from the elliptic potential lemma (Lemma~\ref{lem:sp_elip}) as used in Theorem \ref{thm:main}.
\end{proof}

\subsection{GenOP under linked convex constraints}

\label{sec:gen_oplb}

In this section, we prove regret guarantees of GenOP under linked convex constraints.
Note that this also gives regret bounds of GenOP in the original setting (i.e. given in \eqref{eqn:oplb_reg} in Section \ref{sec:comp}) by taking $n =1$ and $r = b$.
Before giving the regret guarantees, we first give a corollary to Lemma \ref{lem:gamma} that bounds the scaling required to take any point in $\Yc$ in to $\Yc_t^p$.

\begin{corollary}
  \label{cor:gamma}
  Assume the same as Lemma \ref{lem:gamma} and let $\Econf$ hold.
  Let $y$ be any point in $\Yc$ and $\zeta_1 = \max\{ \mu \in [0,1] : \mu y \in \Yc_t^p \}$.
  Then, for all $t$, it holds that
  \begin{equation*}
    \zeta_1 \geq  \frac{r}{r + 2 \sqrt{n} \beta_t \| y \|_{V_t^{-1}}},
  \end{equation*}
  and, with $\bar{y} = \zeta_1 y$, that
  \begin{equation*}
    \zeta_1 \geq 1 -  \frac{2 \sqrt{n}}{r} \beta_t \| \bar{y} \|_{V_t^{-1}}.
  \end{equation*}
\end{corollary}
\begin{proof}
  Conditioned on $\Econf$, it holds that $\Yc \subseteq \Yc_t^o$.
  Therefore, $y$ is in $\Yc_t^o$ and we can apply Lemma \ref{lem:gamma} to get the statement of the corollary.
\end{proof}

With this, we prove the regret bound for GenOP in the following theorem.

\begin{theorem}
  \label{thm:oplb_conv}
  Let Assumptions \ref{ass:set_bound}, \ref{ass:bounded2} and \ref{ass:noise2} hold.
  Then, playing GenOP with $\kappa = 1 + \frac{2 \sqrt{n} S_\theta}{r}$ in the setting with linked convex constraints satisfies
  \begin{equation*}
    R_T \leq  2 \left( 1 + \frac{\sqrt{n} S_\theta}{r} \right) \beta_{T} \sqrt{ 2 d T\log\left(1 + \frac{T}{\lambda d} \right)}
  \end{equation*}
  with probability at least $1 - \delta$.
\end{theorem}
\begin{proof}
  Since $x_*$ is in $\Yc$ and using Corollary \ref{cor:gamma}, we know that $\alpha x_*$ is in $\Yc_t^p$, where $\alpha = \left( \frac{r}{r + 2 \sqrt{n} \beta_t \| x_* \|_{V_t^{-1}}} \right)$.
  Using this, we show that the upper confidence bound of the actions played by GenOP is in fact larger than the optimal reward.

  Since $x_t$ is chosen by a maximization over $\Yc_t^p$, we can reason that
  \begin{align*}
    \hat{\theta}_t^\top x_t + \kappa \beta_t \| x_t \|_{V_t^{-1}} & \geq \alpha \left(  \hat{\theta}_t^\top x_* + \kappa \beta_t \| x_* \|_{V_t^{-1}} \right)\\
    & = \alpha \left(  \theta^\top x_* + (\hat{\theta}_t - \theta)^\top x_* + \kappa \beta_t \| x_* \|_{V_t^{-1}} \right)\\
    & \geq \alpha \left(  \theta^\top x_* - \beta_t \| x_* \|_{V_t^{-1}} + \kappa \beta_t \| x_* \|_{V_t^{-1}} \right)\\
    & = \frac{1}{\frac{2 \sqrt{n} \beta_t \| x_* \|_{V_{t}^{-1}}}{r} + 1} \left(  \theta^\top x_* + \frac{2 \sqrt{n} S_\theta}{r}  \beta_t \| x_* \|_{V_t^{-1}} \right)\\
    & = \frac{\theta^\top x_* + \frac{2 \sqrt{n} S_\theta}{r} \beta_t \| x_* \|_{V_t^{-1}}}{\frac{2 \sqrt{n}}{r} \beta_t \| x_* \|_{V_t^{-1}} + 1}\\
    & \geq \theta^\top x_*,
  \end{align*}
  where the last inequality holds according to the reasoning that
  \begin{align*}
    & \frac{\theta^\top x_* + \frac{2 \sqrt{n} S_\theta}{r} \beta_t \| x_* \|_{V_t^{-1}}}{\frac{2 \sqrt{n}}{r} \beta_t \| x_* \|_{V_t^{-1}} + 1} \geq \theta^\top x_*\\
    & \Longleftrightarrow   \theta^\top x_* + \frac{2 \sqrt{n} S_\theta}{r} \beta_t \| x_* \|_{V_t^{-1}} \geq \theta^\top x_* \left( \frac{2 \sqrt{n}}{r} \beta_t \| x_* \|_{V_t^{-1}} + 1 \right)\\
    & \Longleftrightarrow   \frac{2 \sqrt{n} S_\theta}{r} \beta_t \| x_* \|_{V_t^{-1}} \geq  \theta^\top x_* \frac{2 \sqrt{n}}{r} \beta_t \| x_* \|_{V_t^{-1}}\\
    & \Longleftrightarrow   S_\theta \geq \theta^\top x_*,
  \end{align*}
  which holds by Assumption \ref{ass:set_bound}.
  Therefore, we know that
  \begin{align*}
    r_t & = \theta^\top (x_* - x_t)\\
    & \leq \hat{\theta}_t^\top x_t + \kappa \beta_t \| x_t \|_{V_t^{-1}}  - \theta^\top x_t\\
    & \leq (\hat{\theta}_t - \theta)^\top x_t + \kappa \beta_t \| x_t \|_{V_t^{-1}}\\
    & \leq (1 + \kappa) \beta_t \| x_t \|_{V_t^{-1}}
  \end{align*}
  Applying Cauchy-Schwarz to the cumulative regret and then the elliptic potential lemma (Lemma \ref{lem:sp_elip}) as used in Theorem \ref{thm:main} completes the proof.
\end{proof}

\subsection{Safe-PE under linked convex constraints}

In this section, we give regret bounds for Safe-PE under linked convex constraints.
Let the estimator of each $a_i$ in phase $j$ be $\hat{a}_{j,i}$ and let $\hat{A}_j = [\hat{a}_{j,1}\ ... \ \hat{a}_{j,n}]^\top$.
We then state the specific structural assumption on the noise terms.
\begin{assumption}
  \label{ass:noise3}
    The noise sequences $(\epsilon_t)_{t=1}^T$ and $(\eta_{t,i})_{t=1}^T$ are independent $\rho$-subgaussian random variables for all ${i \in [n]}$.
\end{assumption}
With this, we can then define the confidence set for the parameters in this setting which follows immediately from Lemma \ref{lem:pe_conf}.

\begin{lemma}
  Then, for all $x \in \Xc$ and all $j \in [J]$ it holds that $|x^\top (\hat{\theta}_j - \theta)| \leq \| x \|_{V_j^{-1}} \beta$ and $(\hat{A}_t - A) x \in \| x \|_{V_j^{-1}} \beta \Bb_\infty $ where $\beta = \rho \sqrt{2 \log \left( \frac{4 n k J}{\delta} \right)} + \sqrt{\lambda} S$ with probability at least $1 - \delta$.
\end{lemma}

Then, the only change to the algorithm is the definition of the maximum safe scalings (i.e. line \ref{lem:pe_conf} in Algorithm \ref{alg:pe1}), which is
\begin{equation*}
  \mu_{i,j+1} := \max \bigg\{ \alpha \in [0, \alpha_i] : \alpha \left( u_i^\top \hat{a}_j + \| u_i \|_{V_j^{-1}} \beta_j \Bb_\infty \right) \subseteq \Gc \bigg\}.
\end{equation*}
We then apply Corollary \ref{cor:gamma} to bound the scaling of each direction in the pessimistic set as proven in the following lemma.
Recall the notation from Appendix \ref{apx:safepe}.
\begin{lemma}[Lemma \ref{lem:gam_pe} for linked convex constraints]
  \label{lem:max_scal2}
  Let Assumptions \ref{ass:set_bound}, \ref{ass:bounded} and \ref{ass:fstarconv}  hold.
  For all $i \in [k]$, it holds that $\zeta_{i,j} / \bar{\zeta}_i \geq 1 - \frac{2 \sqrt{n}}{r}  \| v_{i,j} \|_{V_{j-1}^{-1}} \beta$ for all $j \geq 1$.
  Furthermore, $\zeta_{i,j} / \bar{\zeta}_i \geq 1 - \frac{2 n S}{r^2}   \| v_{i,j-1} \|_{V_{j-1}^{-1}} \beta$.
\end{lemma}
\begin{proof}
  Note the similarity between $\mu_{i, j}$ and the definition of $\zeta_1$ in Corollary \ref{cor:gamma}.
  Therefore, we can follow the reasoning of Corollary \ref{cor:gamma} to get that $(\zeta_{i,j} / \bar{\zeta}_i) \geq 1 -  \frac{2 \sqrt{n}}{r} \beta \| v_{i,j} \|_{V_{j - 1}^{-1}}$ and then following the proof of Lemma \ref{lem:gam_pe} gives the claim.
\end{proof}

With this, we can then give the regret bound for Safe-PE in the following theorem.

\begin{theorem}
  Let Assumptions \ref{ass:set_bound}, \ref{ass:bounded}, \ref{ass:pe_noise} and \ref{ass:fstarconv} hold.
  Then, the regret of Safe-PE (Algorithm \ref{alg:pe1}) satisfies
  \begin{equation*}
    R_T  \leq  6 S + 5 \beta \left( \frac{24 S^2 n   }{r^2} + 10 \right) \sqrt{2 d T \log \left(1 + \frac{T}{\lambda d} \right)}
  \end{equation*}
  with probability at least $1 - \delta$.
\end{theorem}

\begin{proof}
  Following the proof of Theorem \ref{thm:safepe} using Lemma \ref{lem:max_scal2} yields the result.
\end{proof}

\section{Problem-dependent analysis of GenOP}

\label{apx:polb}

In this section, we show that the problem dependent analysis approach in Theorem \ref{thm:prob_dep} and Corollary \ref{thm:prob_dep_reg} applies to GenOP.

First, we give some useful facts in the following lemma, which is analogous to Lemma \ref{lem:vt} in the problem-dependent analysis of ROFUL.

\begin{lemma}
\label{lem:vt2}
    Let Assumptions \ref{ass:set_bound} and \ref{ass:bounded} hold, and let $\Econf$ hold.
    Also, let\footnote{Note that this definition of $\zeta_t$ differs from the one used for the problem-dependent analysis of ROFUL in Lemma \ref{lem:vt}.}
    \begin{equation}
    \label{eqn:a_bound2}
        \zeta_t := \max\{ \zeta \geq 0 : \zeta x_t \in \Yc \},
    \end{equation}
    and $v_t = \zeta_t x_t$.
    Then, it follows that:
    \begin{enumerate}
        \item $1/\zeta_t \in \left[1 -  \frac{2}{b} \beta_t \| x_t \|_{V_t^{-1}}, 1 \right]$ \label{it:zeta2}
        \item $\theta^\top (x_t - v_t) \leq \frac{2 S}{b} \beta_t \| x_t \|_{V_t^{-1}}$ \label{it:vt2}
        \item If there exists $\alpha > 0$ such that $x_t = \alpha x_*$, then $v_t = x_*$. \label{it:alph12}
        \item If there does not exists $\alpha > 0$ such that $x_t = \alpha x_*$, then $\theta^\top (x_* - v_t) \geq \Delta$. \label{it:alph22}
    \end{enumerate}
\end{lemma}
\begin{proof}
    We condition on $\Econf$ throughout the proof without further reference.
    We will first give some useful facts.
    In particular, it holds that,
    \begin{equation}
    \label{eqn:ucb_pos2}
        \hat{\theta}_t^\top x_t + \kappa \beta_t \| x_t \|_{V_t^{-1}} \geq \theta^\top x_* > 0
    \end{equation}
    where the first inequality is due to optimism and the second is Assumption \ref{ass:bounded}.
    It follows from \eqref{eqn:ucb_pos2} that $x_t \neq \mathbf{0}$ and therefore the set $\{ \zeta \geq 0 : \zeta x_t \in \Yc \}$ is compact.
    Also, note that $\{ \zeta \geq 0 : \zeta x_t \in \Yc \}$ contains $0$ and is therefore nonempty, so $\zeta_t$ is well-defined.
    Next, we prove each item individually in the following.

    \textbf{\ref{it:zeta2}:} First, we argue that $\zeta_t \geq 1$.
    This holds because $x_t \in \Yc_t^p \subseteq \Yc$ and therefore $1$ is in $\{ \zeta \geq 0 : \zeta x_t \in \Yc \}$.
    It follows that $1/\zeta_t \leq 1$.

    Then, we show that $1/\zeta_t \geq 1 -  \frac{2}{b} \beta_t \| x_t \|_{V_t^{-1}}$.
    In order to do this, we first show that $1/\zeta_t \geq \max \{ \mu \in [0,1] : \mu v_t \in \Yc_t^p \}$.
    Suppose, this was not the case, i.e. that there exists $\mu \in \{ \mu \in [0,1] : \mu v_t \in \Yc_t^p \}$ such that $\mu > 1/\zeta_t$.
    Since, $\mu v_t \in \Yc_t^p$, this would imply that,
    \begin{equation}
        \label{eqn:contra21}
        \hat{\theta}_t^\top x_t + \kappa \beta_t \| x_t \|_{V_t^{-1}} = \max_{x \in \Yc_t^p} \left( \hat{\theta}_t^\top x + \kappa \beta_t \| x \|_{V_t^{-1}} \right) \geq \hat{\theta}_t^\top (\mu v_t) + \kappa \beta_t \| \mu v_t \|_{V_t^{-1}}.
    \end{equation}
    At the same time, given that $v_t = \zeta_t x_t$,
    \begin{equation}
        \label{eqn:contra22}
        \begin{split}
            \hat{\theta}_t^\top x_t + \kappa \beta_t \| x_t \|_{V_t^{-1}} & = (1/\zeta_t) \zeta_t \left( \hat{\theta}_t^\top x_t + \kappa \beta_t \| x_t \|_{V_t^{-1}} \right)\\
            & < \mu \zeta_t \left( \hat{\theta}_t^\top x_t + \kappa \beta_t \| x_t \|_{V_t^{-1}} \right)\\
            & = \mu \left( \hat{\theta}_t^\top v_t + \kappa \beta_t \| v_t \|_{V_t^{-1}} \right)\\
            & = \hat{\theta}_t^\top (\mu v_t) + \kappa \beta_t \| \mu v_t \|_{V_t^{-1}},
        \end{split}
    \end{equation}
    where the inequality uses \eqref{eqn:ucb_pos2}.
    Since \eqref{eqn:contra21} and \eqref{eqn:contra22} cannot simultaneously hold, it follows that $1/\zeta_t \geq \max \{ \mu \in [0,1] : \mu v_t \in \Yc_t^p \}$.
    
    Finally, using Corollary \ref{cor:gamma} by taking $n = 1$ and $r = b$ (since the setting with linked convex constraints is a more general case), we have that
    \begin{equation*}
        1/\zeta_t \geq \max \{ \mu \in [0,1] : \mu v_t \in \Yc_t^p \} \geq 1 -  \frac{2}{b} \beta_t \| x_t \|_{V_t^{-1}}.
    \end{equation*}

    \textbf{\ref{it:vt2}:}
    Since, $v_t = \zeta_t x_t$ and $1/\zeta_t \in \left[1 -  \frac{2}{b} \beta_t \| x_t \|_{V_t^{-1}}, 1 \right]$, it holds that
    \begin{equation*}
      \theta^\top (x_t - v_t) = \theta^\top v_t (1/\zeta_t - 1) \leq S |1/\zeta_t - 1| = S (1 - 1/\zeta_t) \leq \frac{2 S}{b} \beta_t \| x_t \|_{V_t^{-1}}.
    \end{equation*}

    \textbf{\ref{it:alph12}:}
    From Lemma \ref{lem:vt}, we know that $\max\{ \zeta \geq 0 : \zeta x_* \in \Yc \} = 1$.
    Then, if there exists $\alpha > 0$ such that $x_t = \alpha x_*$,
    \begin{equation*}
        \zeta_t = \max\{ \zeta \geq 0 : \zeta x_t \in \Yc \} = \frac{1}{\alpha} \max\{ \tilde{\zeta} \geq 0 : \tilde{\zeta} x_* \in \Yc \} = \frac{1}{\alpha},
    \end{equation*}
    where we use the mapping $\tilde{\zeta} = \alpha \zeta'$.
    Therefore, it follows that
    \begin{equation*}
        v_t = \zeta_t x_t = \alpha \zeta_t x_* = x_*.
    \end{equation*}

    \textbf{\ref{it:alph22}:} First, note that if there does not exist $\alpha > 0$ such that $x_t = \alpha x_*$, then there does not exist $\alpha' > 0$ such that $v_t = \alpha' x_*$ as $v_t = \zeta_t x_t$.
    Then, since $v_t \in \Yc$, it follows from the definition of $\Delta$ that,
    \begin{equation*}
        \Delta = \inf_{x \in \Yc:\ x \neq \alpha x_*\ \forall \alpha > 0 } \theta^\top (x_* - x) \leq \theta^\top (x_* - v_t).
    \end{equation*}
\end{proof}

\begin{theorem}
  Let Assumptions \ref{ass:set_bound}, \ref{ass:bounded} and \ref{ass:noise} hold.
  If $\Delta > 0$, then the number of wrong directions chosen by GenOP with $\kappa = 1 + \frac{2 S}{b}$ (defined by \eqref{eqn:oplb}) satisfies
  \begin{equation*}
    B_T \leq \frac{1}{\Delta^2} 8 d \left( 1 + \frac{2 S}{b} \right)^2 \beta_T^2 \log\left(1 + \frac{T}{\lambda d} \right)
  \end{equation*}
  with probability at least $1 - \delta$.
\end{theorem}
\begin{proof}
  We condition on $\Econf$ defined in \eqref{eqn:econf} without further mention.
  Consider the the instantaneous directional regret,
  \begin{align*}
    \tilde{r}_t & = \theta^\top (x^* - v_t)\\
    & \leq \hat{\theta}_t^\top x_t + \kappa \beta_t \| x_t \|_{V_t^{-1}} - \theta^\top v_t \\
    & = \theta^\top x_t + (\theta - \hat{\theta}_t)^\top x_t + \kappa \beta_t \| x_t \|_{V_t^{-1}} - \theta^\top v_t \\
    & \leq \theta^\top (x_t - v_t) + (\kappa + 1) \beta_t \| x_t \|_{V_t^{-1}}\\
    & \leq \frac{2 S}{b} \beta_t \| x_t \|_{V_t^{-1}} + (\kappa + 1) \beta_t \| x_t \|_{V_t^{-1}}\\
    & \leq 2(1 + \frac{2 S}{b}) \beta_t \| x_t \|_{V_t^{-1}} ,
  \end{align*}
  where the first inequality uses optimism (with $\kappa = 1 + \frac{2 S}{b}$), the second inequality uses the definition of the confidence set, the third inequality uses Lemma \ref{lem:vt2} (\#\ref{it:vt2}).
  Then, from Lemma \ref{lem:vt2} (\#\ref{it:alph12}, \#\ref{it:alph22}), we know that either $\rtil_t = 0$ if there exists $\alpha > 0$ such that $x_t = \alpha x_*$ or $\rtil_t \geq \Delta$ otherwise.
  Then, using the bound $B_T \leq \tilde{R}_T / \Delta$ and the fact that $\rtil_t \leq \rtil_t^2 / \Delta$, we have that
  \begin{equation*}
    B_T \leq \frac{\tilde{R}_T}{\Delta} = \frac{1}{\Delta} \sum_{t=1}^T \tilde{r}_t \leq \frac{1}{\Delta^2} \sum_{t=1}^{T} (\rtil_t)^2 = \frac{4}{\Delta^2} \left(1 + \frac{2 S}{b} \right)^2 \beta_T^2 \sum_{t=1}^{T}  \| x_t \|_{V_t^{-1}}^2 \leq \frac{1}{\Delta^2} 8 d \left(1 + \frac{2 S}{b} \right)^2 \beta_T^2 \log\left(1 + \frac{T}{\lambda d} \right),
  \end{equation*}
  where the last inequality comes from Lemma \ref{lem:sp_elip}.
\end{proof}

It then immediately follows that it is possible to achieve regret that only depends on $d$ in $\Oc(\polylog(T))$ terms with the same reasoning as Corollary \ref{thm:prob_dep_reg}.

\begin{corollary}
  Let Assumptions \ref{ass:set_bound}, \ref{ass:bounded} and \ref{ass:noise} hold.
  If $\Delta > 0$, consider the algorithm:
  \begin{enumerate}
    \item Play GenOP until any single direction has been played more than $\bar{B} := \frac{1}{\Delta^2} 8 d \left( 1 + \frac{2 S}{b} \right)^2 \beta_T^2 \log\left(1 + \frac{T}{\lambda d} \right)$ times. Let this direction be denoted by $u_*$.
    \item For the remaining rounds, play GenOP (after restarting) for the 1-dimensional safe linear bandit problem of choosing $\xi_t \in \Rb_+$ and then playing~$\xi_t u_*$.
  \end{enumerate}
  Then, with probability at least $1 - 2 \delta$,
  \begin{equation*}
      R_T \leq \frac{8 S}{b} \beta_{2 \bar{B}} \sqrt{ d  \bar{B} \log\left(1 + \frac{2 \bar{B}}{\lambda d} \right)} + \frac{4 S}{b} \tilde{\beta}_{T} \sqrt{ 2 T\log\left(1 + \frac{T}{\lambda d} \right)}
  \end{equation*}
  where $\tilde{\beta}_{T}$ is $\beta_T$ with $d = 1$.
\end{corollary}

\section{Details on numerical experiments}
\label{apx:num_exp}

In this section, we give the details of the numerical experiments that were not included in the body of the paper as well as details on the computing setup and additional results.
The computing hardware specifications are given in Section \ref{sec:comp_hard}.
Then, the details on the simulation results for the settings with linear contraints, linked convex constraints, and star convex multi-armed bandit are given in Sections \ref{sec:lin_const}, \ref{sec:link_const}, and \ref{sec:fin_star}, respectively.

\subsection{Computing hardware}

\label{sec:comp_hard}

All simulations were run on a Lenovo ThinkPad T470 with an Intel Core i7 processor and 16 GB of memory.

\subsection{Linear constraints}

\label{sec:lin_const}

In the first setting (results in Figure \ref{fig:sims:a}), we take $d = 2$ and $T = 5 \times 10^4$, and also take the action set to be a finite star convex set $\Xc = \bigcup_{i \in [10]} \{ \alpha  u_k : \alpha \in [0,1] \}$.
For each trial, we sample $\theta \sim \Uc(\Bb_{\infty})$, $u_k \sim \Uc(\Sb)$ for all $k \in [10]$ (where we resample $\{ u_k \}_k$ until it holds that $\theta^\top x_* > 0$), $b \sim \Uc[0.25,1]$, $a \sim \Uc(\Bb_{\infty})$.
The learner is only given the prior information on these parameters that $\| \theta \| \leq \sqrt{2}$ and $\| a \| \leq \sqrt{2}$.
As such, the algorithm can take $S_a = S_\theta = \sqrt{2}$.
The noise terms are sampled i.i.d as $\eta_t \sim \Nc(\sigma)$ and $\epsilon_t \sim \Nc(\sigma)$, where $\sigma = 0.1$.
The learner is given $\sigma$.
For the regularization parameter, we use $\lambda = 1$ for all algorithms tested.

The second setting (results in Figure \ref{fig:sims:b}) is the same except that $T = 1 \times 10^5$ and $b \sim [0.05,0.25]$.

To implement the algorithms in this setting, we enumerate over the directions in order to calculate the algorithms' updates.
To show how this matches the specified update for GenOP, consider the update specified in \eqref{eqn:oplb},
\begin{equation}
\label{eqn:genop_max}
    \argmax_{x \in \Yc_t^p} \left( \hat{\theta}_{t}^\top x + \kappa \beta_{t} \| x \|_{V_{t}^{-1}} \right).
\end{equation}
Since the pessimistic set can be defined as
\begin{equation*}
    \Yc_t^p = \bigcup_{i \in [10]} \left\{ \alpha  u_i : \alpha \in [0,1], \alpha \left( \hat{a}_t^\top u_i + \beta_t \| u_i \|_{V_t^{-1}} \right) \leq b \right\},
\end{equation*}
and the objective is convex, we know that for each of the line segments, the maximum objective value is attained at the origin or the maximum scaling in that direction.
Therefore, we find a point in \eqref{eqn:genop_max} by optimizing over these points, i.e.
\begin{equation*}
    x_t \in \argmax_{x \in \{ \mathbf{0}, \zeta^p_1 u_1, ..., \zeta^p_{10} u_{10} \}} \left( \hat{\theta}_{t}^\top x + \kappa \beta_{t} \| x \|_{V_{t}^{-1}} \right),
\end{equation*}
where
\begin{equation*}
    \zeta^p_i =  \begin{cases}
     \min(b/(\hat{a}_t^\top u_i + \beta_t \| u_i \|_{V_t^{-1}}), 1), & \text{ if } \hat{a}_t^\top u_i + \beta_t \| u_i \|_{V_t^{-1}} > 0\\
    1 & \text{ else. }
    \end{cases}
\end{equation*}

The calculation used for ROFUL uses a similar idea.
In particular, the optimistic action (line \ref{lne:opt_act} of Algorithm \ref{alg:main_alg}) is calculated as
\begin{equation*}
    \xtil_t \in \argmax_{x \in \{ \mathbf{0}, \zeta^o_1 u_1, ..., \zeta^o_{10} u_{10} \}} \left( \hat{\theta}_{t}^\top x + \beta_{t} \| x \|_{V_{t}^{-1}} \right),
\end{equation*}
where
\begin{equation*}
    \zeta^o_i =  \begin{cases}
     \min(b/(\hat{a}_t^\top u_i - \beta_t \| u_i \|_{V_t^{-1}}), 1), & \text{ if } \hat{a}_t^\top u_i - \beta_t \| u_i \|_{V_t^{-1}} > 0\\
    1 & \text{ else. }
    \end{cases}
\end{equation*}
Then, $\mu_t$ from line \ref{lne:scale} is calculated as
\begin{equation}
\label{eqn:zeta_p}
    \mu_t =  \begin{cases}
     \min(b/(\hat{a}_t^\top \xtil_t + \beta_t \| \xtil_t \|_{V_t^{-1}}), 1), & \text{ if } \hat{a}_t^\top \xtil_t + \beta_t \| \xtil_t \|_{V_t^{-1}} > 0\\
     1 & \text{ else. }
    \end{cases}
\end{equation}

\subsection{Linked convex constraints}

\label{sec:link_const}

In this (results Figure \ref{fig:sims:b}), $d=2$, $n=2$, $\Xc = \Bb$, and $T = 10^5$.
We take $\Gc = b \Bb$, where $b \sim \Uc[0.25,1]$ for each trial.
The constraint matrix and reward vector are randomly sampled, where each row of $A$ is sampled as $a_i \sim \Uc(\Bb_{\infty})$ for all $i \in [n]$ and $\theta \sim \Uc(\Bb_{\infty})$.
The learner is only given the prior information on these parameters that $\| \theta \| \leq \sqrt{2}$ and $\| a_i \| \leq \sqrt{2}$ for all $i \in [n]$.
As such, the algorithms can take $S = \sqrt{2}$ and $r_1 = 0.25$.
The noise terms are sampled i.i.d as $\eta_t \sim \Nc(\sigma I)$ and $\epsilon_t \sim \Nc(\sigma)$, where $\sigma = 0.1$.
The learner is given $\sigma$.
For the regularization parameter, we use $\lambda = 1$ for both algorithms.
We simulate this setting for 30 trials, where different realizations of the problem parameters are used for each trial.
In Figure \ref{fig:sims:b}, the mean of the regret at each round $t$ normalized by square-root $t$ is shown along with the plus-or-minus one standard deviation.

For this setting, we relax the optimistic and pessimistic sets such that they use the 2-norm ball instead of the infinity-ball.
In particular, we use the sets
\begin{equation*}
  \Yc_t^o = \{ x \in \Xc : \hat{A}_t x + \sqrt{n} \beta_t \| x \|_{V_t^{-1}} \Bb  \cap \Gc \neq \emptyset \}
\end{equation*}
and
\begin{equation*}
  \Yc_t^p = \{ x \in \Xc : \hat{A}_t x + \sqrt{n} \beta_t \| x \|_{V_t^{-1}} \Bb  \subseteq \Gc \}.
\end{equation*}

\subsubsection{Additional Results}

\begin{figure}[t]
  \centering{
  \includegraphics[width=0.5\columnwidth]{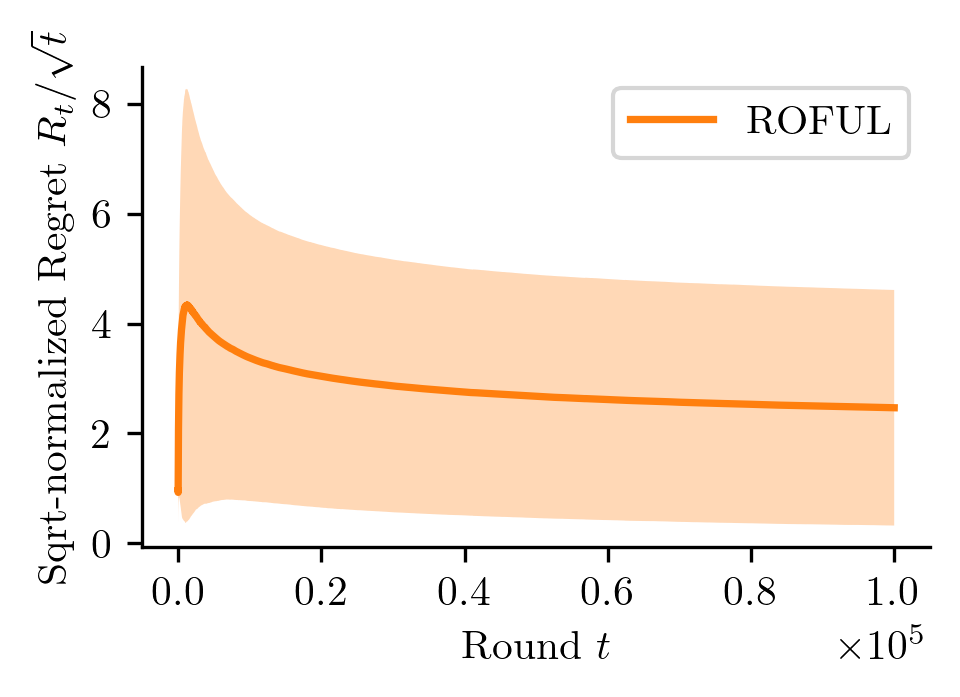}
  }
  \caption{Simulation results for setting with linked convex constraints and box constraints.}
  \label{fig:box}
\end{figure}

We also includes additional results (given in Figure \ref{fig:box}), where everything is the same except that $\Gc = b \Bb_{\infty}$ and $\Xc = \bigcup_{k \in [10]} \{ \alpha  u_k : \alpha \in [0,1] \}$ where $u_k \sim \Uc(\Sb)$ for all $k \in [10]$.
In Figure \ref{fig:box}, the mean of the regret at each round $t$ normalized by square-root $t$ is shown along with plus-or-minus one standard deviation.

\subsection{Star convex multi-armed bandit}

\label{sec:fin_star}

In this setting (results in Figure \ref{fig:sims:d}), the action set only consists of the coordinate directions with scalings between $0$ and $1$.
We set $\theta = a = [1\ 0\ ...\ 0]^\top$ and $b = 0.5$ and use i.i.d. Gaussian noise of standard deviation $0.1$.
In this case, we only give the learner the knowledge that $\| a \|, \| \theta \| \leq S = 2$ because if the learner was given the information that $\| a \| \leq 1$, it would be initially known that the optimal action satisfies the constraint.
In this setting, we simulate ROFUL and Safe-PE for $3$ trials for $d = 10$.

\subsubsection{Additional results}

\begin{figure}[t]
  \centering{
  \includegraphics[width=0.5\columnwidth]{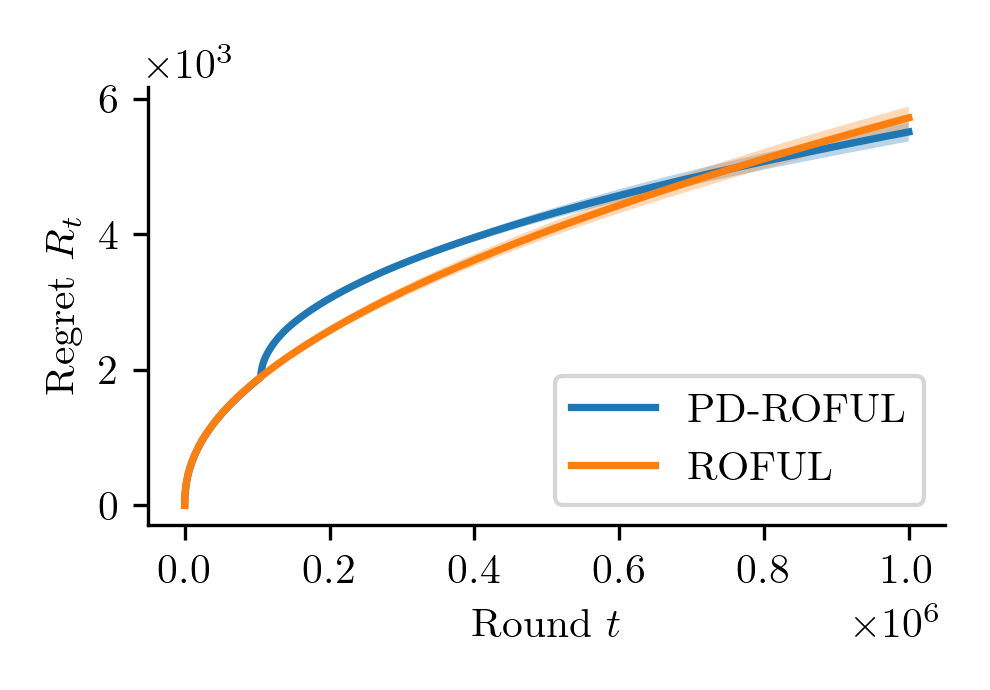}
  }
  \caption{Simulation results for ROFUL and PD-ROFUL in problem with $\Delta > 0$ and known.}
  \label{fig:pd_roful}
\end{figure}

We also simulate PD-ROFUL (Algorithm \ref{alg:pd_oful}) and ROFUL in the same setting with, except with $b = 0.9$ and $S = 1.5$.
We make this modification to make the initial phase of duration less than $T$ so that there is a difference between ROFUL and PD-ROFUL.
We simulate both algorithms for $5$ trials with $d = 10$ and show the mean and standard deviation of the regret in Figure \ref{fig:pd_roful}.

\end{document}